\documentclass[lettersize,journal]{IEEEtran}
\usepackage[utf8]{inputenc}

\usepackage{textcomp}

\usepackage{amsmath,amsfonts,amssymb}
\usepackage{amsthm}
\usepackage{algorithm}
\usepackage{algorithmic}
\usepackage{array}
\usepackage{subcaption}
\usepackage{textcomp}
\usepackage{stfloats}
\usepackage{url}
\usepackage{verbatim}
\usepackage{graphicx}
\usepackage{cite}
\usepackage{booktabs}
\usepackage[dvipsnames,table,xcdraw]{xcolor}
\usepackage{pifont}
\newcommand{\cmark}{\ding{51}}
\newcommand{\xmark}{\ding{55}}

\newcommand{\gray}[1]{\textcolor{gray}{#1}}
\definecolor{deepred}{rgb}{0.6,0,0}
\newtheorem{assumption}{Assumption}

\newtheorem{theorem}{Theorem}
\newtheorem{corollary}{Corollary}
\newtheorem{lemma}{Lemma}
\newtheorem{definition}{Definition}
\newtheorem{remark}{Remark}

\begin{document}

\title{Lifecycle-Aware Federated Continual Learning in Mobile Autonomous Systems}

\author{Beining Wu,~\IEEEmembership{Member,~IEEE,} and
         Jun Huang,~\IEEEmembership{Senior Member,~IEEE}
\thanks{Beining Wu and Jun Huang are with the Department of Electrical Engineering and Computer Science, South Dakota State University, Brookings, 57006 SD. Email: Wu.Beining@jacks.sdstate.edu., Jun.Huang@sdstate.edu.}
\thanks{Manuscript received [Date]; revised [Date].}}



\maketitle

\begin{abstract}
Federated continual learning (FCL) allows distributed autonomous fleets to adapt collaboratively to evolving terrain types across extended mission lifecycles. However, current approaches face several key challenges: 1) they use uniform protection strategies that do not account for the varying sensitivities to forgetting on different network layers; 2) they focus primarily on preventing forgetting during training, without addressing the long-term effects of cumulative drift; and 3) they often depend on idealized simulations that fail to capture the real-world heterogeneity present in distributed fleets. In this paper, we propose a lifecycle-aware dual-timescale FCL framework that incorporates training-time (pre-forgetting) prevention and (post-forgetting) recovery. Under this framework, we design a layer-selective rehearsal strategy that mitigates immediate forgetting during local training, and a rapid knowledge recovery strategy that restores degraded models after long-term cumulative drift. We present a theoretical analysis that characterizes heterogeneous forgetting dynamics and establishes the inevitability of long-term degradation. Our experimental results show that this framework achieves up to 8.3\% mIoU improvement over the strongest federated baseline and up to 31.7\% over conventional fine-tuning. We also deploy the FCL framework on a real-world rover testbed to assess system-level robustness under realistic constraints; the testing results further confirm the effectiveness of our FCL design.
\end{abstract}


\section{Introduction}
\label{sec:introduction}

\IEEEPARstart{M}{obile} autonomous fleets are deployed across a range of domains where environments evolve throughout the mission lifecycle, from polar research stations and precision agriculture to planetary exploration~\cite{Wu2025RACS, Wu2025WASA, Ding2025IPCCC}. Among these, planetary exploration missions, especially those on Mars, represent the most demanding instance: they have grown from single probes into coordinated fleets of autonomous rovers that work together to operate across geologically diverse regions~\cite{NASA_MER_Spirit_Opportunity, NASA2025MarsReport}. In these missions, semantic terrain segmentation plays a vital role in guiding rovers safely across unfamiliar regions and helping them recognize sites of scientific interest~\cite{Sun2023JGR}. Because each rover acquires imagery in a different region, the data observed across the fleet is inherently heterogeneous, and local onboard decision-making requires models that remain up to date with each rover's own operating conditions. At the same time, the bandwidth and delay limits of the Interplanetary Internet~\cite{Khoroshylov2021SST, Wu2026COMST} make it impractical to return large volumes of raw data to Earth for centralized training. Federated Learning (FL) has therefore become a natural architectural choice~\cite{McMahan2017AISTATS, Li2020MLSys, Wu2026MNET, Wu2026ARXIV}, allowing rover fleets to train shared segmentation models collaboratively while exchanging only model parameters. Yet the reality of long journeys through space reveals a gap in current methods: models must evolve and adapt throughout a long-term mission lifecycle, rather than remain unchanged after their first deployment.

As rovers travel through diverse regions, they encounter terrain types that were absent in earlier exploration zones, such as shifts from loose soil to exposed bedrock. In such settings, the fleet must learn new knowledge step by step while retaining to what it has already learned~\cite{Wang2024TPAMI}. However, continual adaptation on resource-limited edge devices often leads to \emph{catastrophic forgetting}. As new knowledge is absorbed, changes to model parameters can cause earlier knowledge to fade and be lost~\cite{Rebuffi2017CVPR, Douillard2021CVPR}.  
In FL settings, such a degradation typically occurs at two timescales.
At the training-time scale, rovers face immediate conflicts among gradients during local updates, as they seek to learn new tasks without weakening their performance on prior ones~\cite{Li2025CST}. At the long-term scale, repeated federated aggregation of Non-IID models throughout the mission introduces cumulative drift, which slowly erodes the model’s ability to recognize classes learned in earlier stages~\cite{Wang2024NeurIPS, Yang2024TKDE}. These challenges call for federated continual learning (FCL), which enables distributed rover fleets to learn over time without forgetting what was learned before. We term a framework \emph{lifecycle-aware} when it explicitly coordinates defenses across both temporal phases, coupling training-time forgetting mitigation with long-term degradation recovery through shared operational state.

While FCL has been recently investigated in the AI communities~\cite{Rebuffi2017CVPR,Douillard2021CVPR,Wang2022CVPR,Zhao2023ICML,Dong2023CVPR,Li2023TPAMI,Qiao2024NeurIPS,Wang2024NeurIPS,He2025CVPR,Wei2025WACV,Behrouz2025NeurIPS,Li2025CST,Batra2024ARXIV}, existing research presents significant limitations in the following aspects. 
%
First, prior rehearsal-based strategies (replay while learning new) ignore the heterogeneous forgetting sensitivities of deep network layers~\cite{Rebuffi2017CVPR, Douillard2021CVPR}. Although recent studies have begun to examine layer-wise sensitivity~\cite{Qiao2024NeurIPS, He2025CVPR, Wei2025WACV}, they still apply uniform protection strategies to all layers. Such a ``one-size-fits-all'' approach neglects the fact that different layers have distinct \emph{plasticity} and \emph{stability} requirements. Backbone layers benefit from continued adaptation as representations evolve, whereas classifier layers require strong stabilization to preserve decision boundaries~\cite{Behrouz2025NeurIPS}. As a result, existing rehearsal-based frameworks with uniform or static layer-wise protection cause harmful effects: over-protected layers that should remain adaptable, while under-protected layers that are more prone to forgetting~\cite{Zhao2023ICML}.

Second, existing recovery mechanisms, including Low-Rank Adaptation~\cite{He2025CVPR, Wei2025WACV}, Knowledge Distillation~\cite{Wang2022CVPR}, and Variational Inference~\cite{Li2023TPAMI, Batra2024ARXIV}, depend heavily on iterative gradient-based optimization. Such methods often require multiple training epochs to converge, making them computationally expensive and vulnerable to ongoing multiplicative drift in federated learning~\cite{Wang2024NeurIPS}. Consequently, prior recovery approaches lack fast mechanisms capable of restoring degraded models with low computational and communication overhead.

Third, although several theoretical FCL frameworks have been proposed~\cite{Dong2023CVPR, Li2025CST}, most current studies are based on simulations that do not account for the strict system heterogeneity in real-world rover fleets. For example, these simulations often overlook differences in onboard storage capacity and the impact of limited communication resources. Therefore, solutions developed through simulation alone cannot guarantee reliable mission deployment in practice.

In a nutshell, prior studies on FCL present the following significant research gaps:

\begin{enumerate}
    \item Current rehearsal-based approaches that employ uniform or static layer-wise protection often result in over-protected layers requiring adaptability, while under-protected layers are more prone to forgetting.

    \item Previous recovery methods exhibit slow convergence and are unable to restore degraded models while maintaining low computational and communication overhead.

    \item Most existing studies are based on simulations that fail to account for the system heterogeneity in real-world rover fleets. Solutions developed solely through simulation cannot ensure reliable mission deployment in practice.
\end{enumerate}

To close the above research gaps, this work aims to design a lifecycle-aware FCL framework that leverages stratified rehearsal for short-term learning and meta-learned recovery for long-term learning to combat catastrophic forgetting. Specifically, we make the following contributions.

\begin{itemize}
    \item We propose a dual-timescale FCL framework that incorporates training-time (pre-forgetting) prevention and (post-forgetting) recovery. Under this framework, we design a Layer-Selective Rehearsal (LSR) strategy that mitigates immediate forgetting during local training, and a Rapid Knowledge Recovery (RKR) strategy that actively restores degraded models after long-term cumulative drift.

    \item We present a theoretical analysis that characterizes heterogeneous forgetting dynamics over deep network layers and demonstrates the inevitability of long-term degradation. We also provide a formal justification for the LSR protection design and derive sample-complexity bounds for the performance of RKR compared with retraining.

    \item We evaluate our FCL framework on three representative Mars terrain datasets (MarsScapes, S5Mars, AI4MARS), achieving up to 8.3\% mIoU improvement over the strongest federated baseline and up to 31.7\% over conventional fine-tuning. We then deploy the FCL framework on a real-world rover testbed to assess system-level robustness under realistic constraints, including extreme Non-IID conditions and resource heterogeneity. The testing results further confirm the effectiveness of our FCL design.


\end{itemize}

The remainder of this paper is organized as follows. 
Section~\ref{sec:related_work} briefly summarizes the related studies. 
Section~\ref{sec:problem_formulation} describes the system model and formulates the federated class-incremental learning problem. 
In Section~\ref{sec:proposed_method}, we present our proposed lifecycle-aware defense framework. 
Section~\ref{sec:theoretical_analysis} provides the theoretical analysis of forgetting dynamics and recovery bounds. 
We discuss the experimental results from both simulations and the physical testbed in Section~\ref{sec:experiments}, and draw the conclusion in Section~\ref{sec:conclusion}.

\section{Related Work}
\label{sec:related_work}

\subsection{Federated Learning in MAS}
For collaborative learning in mobile autonomous systems (MAS), federated learning (FL) trains a decentralized global model by aggregating network parameters from distributed autonomous agents. McMahan \emph{et al.} \cite{McMahan2017AISTATS} propose a communication-efficient aggregation strategy that enables collaborative training across edge devices without exchanging raw data. Li \emph{et al.} \cite{Li2020MLSys} develop FedProx to address data heterogeneity across federated agents via introducing proximal terms into the local optimization objective. Recent works \cite{Wu2025TON, Cai2025TITS, Huang2025TMC, Wu2025WASA, Ding2025IPCCC, Wu2025RACS, Pudasaini2026HPSR} apply FL to autonomous systems, including vehicular platooning, connected autonomous vehicles, UAV trajectory planning, dual-level UAV--vehicle collaborative learning, precision-agriculture path planning, and communication-efficient federated unlearning for smart agriculture, where distributed agents must collaboratively learn models under communication constraints and extreme data heterogeneity. A complementary line of work on edge collaborative perception further studies distributed online learning for edge video analytics~\cite{Fang2025TON}, task-oriented collaborative perception~\cite{Fang2025JSAC}, and edge--aerial task-oriented communication for low-altitude navigation~\cite{Fang2026GLOBECOM}, while FL on heterogeneous IoT edges under missing modalities has also been investigated~\cite{Wu2026ARXIV}. At the wireless and coordination substrate, parallel studies on AoI-aware DRL-based resource allocation~\cite{Wu2023ACCESS}, SWIPT-assisted D2D energy harvesting~\cite{Xing2026ACR}, fault-tolerant nano-communication switching~\cite{Pan2023SCIS}, model-free cooperative optimal output regulation for multi-agent systems~\cite{Wu2023MPE}, and multi-agent reinforcement learning for Stackelberg security games in MEC~\cite{Ding2026ICNC} underpin the networking and control layer on which such federated autonomous systems operate. Similar challenges arise in planetary exploration missions where rover fleets must collaboratively learn terrain segmentation models across geologically diverse regions while preserving data locality.

To enable learning new knowledge continuously in federated settings, He \emph{et al.} \cite{He2025TMC} propose federated continual learning via server-side generative replay, though the method assumes identifiable task transitions and focuses on bounding immediate forgetting within individual task stages. Dong \emph{et al.} \cite{Dong2023CVPR} address federated incremental semantic segmentation by storing old global models for immediate distillation. 

The above FL methods~\cite{McMahan2017AISTATS, Li2020MLSys, He2025TMC, Dong2023CVPR} assume static deployment and are evaluated at task completion, whereas missions require models to adapt throughout the operational lifecycle across multiple temporal scales.

A parallel line of work studies learned terrain analysis for Mars rovers. SPOC~\cite{Rothrock2016AIAA} introduced CNN-based Mars terrain classification and was later integrated into the rover-operations analytics pipeline MAARS~\cite{Ono2020SMC}. MLNav~\cite{Daftry2022RAL} studies onboard learned navigation under flight-software constraints. Several terrain datasets have been released by the community, including AI4MARS~\cite{Swan2021CVPR}, MarsScapes~\cite{Liu2023AA}, and S5Mars~\cite{Zhang2024TGRS}, which we use in our evaluation. These works share an offline, single-rover, non-incremental training assumption: data is collected in situ, returned to Earth for labeling, and used to train a static model that is subsequently uploaded back to the rover. The federated, class-incremental regime that arises over a multi-year mission lifecycle is left open, which is the gap our work targets.

\subsection{Continual Learning}
Continual learning (a.k.a. incremental learning or lifelong learning) focuses on identifying new categories continuously in dynamic real-world scenarios \cite{Li2025CST, Wu2026TNSE}. A major challenge is catastrophic forgetting, where models learn new categories sequentially and have limited memory for storing old training data. Existing methods fall into regularization-based approaches, generative replay, and exemplar memory construction \cite{Rebuffi2017CVPR}. Rebuffi \emph{et al.} \cite{Rebuffi2017CVPR} propose iCaRL that maintains a fixed-size exemplar buffer to rehearse old classes during new task training. Buzzega \emph{et al.} \cite{Buzzega2020NeurIPS} design dark experience replay to store and replay past model predictions. Beyond sample-level rehearsal, recent work on shared spatial memory through predictive coding explores representation-level consolidation mechanisms that complement exemplar-based replay~\cite{Fang2025ARXIV}. Recent advances \cite{Douillard2021CVPR, Lin2025CVPR} extend incremental learning to semantic segmentation tasks, addressing challenges such as background shift and spatial consistency in dense prediction scenarios. Additional works \cite{Cong2024ECCV, Fang2025CVPR, Yang2025ICLR, Yin2025CVPR} further develop adaptive strategies for conflict mitigation and instance replay mechanisms.
Recent theoretical analyses \cite{Wang2024NeurIPS, Khademi2024NeurIPS} reveal that catastrophic forgetting exhibits complex dynamics beyond immediate performance degradation, with model myopia and task confusion accumulating inevitably across extended task sequences. Empirical studies \cite{Zhao2023ICML} further show that different network modules exhibit heterogeneous forgetting patterns, with only a few task-specific parameters sensitive to distribution shifts, while others can be shared across tasks. Layer-selective protection in CL has been investigated along two directions. The first applies uniform or pre-specified structural protection, including binary freeze/adapt masks~\cite{Zhao2023ICML} and fixed low-rank adaptation modules at chosen layers~\cite{Qiao2024NeurIPS}. The per-layer allocation in these methods is fixed at the outset and does not change with task dynamics. The second applies post-hoc sensitivity analysis: a one-shot importance score identifies sensitive parameters after a task is completed, and rehearsal is then applied over the selected subset~\cite{Zhao2023ICML}. 

These methods~\cite{Zhao2023ICML, Qiao2024NeurIPS} target centralized settings. Prior incremental learning approaches~\cite{Rebuffi2017CVPR, Douillard2021CVPR} evaluate forgetting at task completion and do not account for the cumulative degradation induced by repeated federated aggregation under extreme Non-IID conditions.

\subsection{Lifecycle-Aware Knowledge Management}
In long-term scenarios, knowledge degradation across extended learning sequences remains a challenge despite training-time defense mechanisms. Recent theoretical work \cite{Wang2024NeurIPS} reveals that catastrophic forgetting accumulates inevitably as incremental errors compound across sequential tasks. Memory buffer coverage limitations exacerbate this degradation as the number of learned classes grows \cite{Khademi2024NeurIPS}. A trivial solution to recover from such degradation is to periodically retrain the entire model on memory buffers, but this approach incurs prohibitive computational overhead and communication costs in resource-constrained autonomous systems. To enable efficient knowledge recovery, meta-learning-based approaches \cite{Finn2017ICML} learn transferable adaptation patterns that generalize across tasks. Wang \emph{et al.} \cite{Wang2022CVPR} employ episodic replay distillation that simulates task transitions during meta-training, enabling rapid few-shot adaptation through learned correction patterns. Recent works \cite{He2025CVPR, Li2023TPAMI} further improve recovery efficiency through parameter-efficient strategies such as low-rank adaptation and variational distillation.

These methods~\cite{Wang2022CVPR, He2025CVPR, Li2023TPAMI} assume centralized settings with full model access and do not address the multiplicative degradation induced by repeated federated aggregation under Non-IID conditions. They also treat prevention and recovery as isolated operations rather than coordinating them across temporal scales.
\section{System Model and Problem Formulation}
\label{sec:problem_formulation}

\subsection{System Model}

In a long-term Mars exploration mission, a fleet of $K$ autonomous rovers, denoted as $\{\mathcal{C}_k\}_{k=1}^K$, is deployed to conduct scientific investigations under the \emph{wireless federated learning} framework. Each rover is responsible for collecting data and performing local learning, while a central server $\{\mathcal{S}\}$ coordinates collaborative model training through periodic synchronization. Each rover is constrained by its onboard storage capacity $\mathcal{C}_k$ and, as such, can retain only a limited volume of historical data for training. The federated architecture adopted here is a consequence of the communication budget rather than a privacy requirement. Even if the aggregation server were placed on Mars rather than on Earth, the relay and surface links between rovers share the same order of bandwidth as the Mars--Earth relay path~\cite{Edwards2006AA}, so raw imagery from multiple rovers cannot be collected at a single node; only model parameters, which constitute a far smaller per-round payload, are compatible with the available link budget.

As the mission progresses and rovers operate in new territories, they encounter previously unseen phenomena and knowledge not present in earlier regions. We model this process as a sequence of learning tasks ${\mathcal{T}^t}_{t=1}^T$, where each task $\mathcal{T}^t$ corresponds to data collected in a newly explored region. Since future tasks may revisit environments where earlier knowledge remains relevant, the learning system must preserve acquired knowledge while incorporating new information, avoiding knowledge override and catastrophic forgetting.

The data observed by each rover is inherently heterogeneous due to spatially varying geological and environmental factors. Specifically, each rover $\mathcal{C}_k$ samples data from its local distribution $\mathcal{P}_k$, and the set of distributions ${\mathcal{P}_k}$ is non-independent and non-identically distributed (Non-IID) across the fleet. 

These characteristics give rise to learning challenges at two different time scales. At the training-time scale, rovers face immediate optimization challenges in each task as they balance learning new knowledge with preserving existing knowledge through local gradient updates. At the long-term scale, the cumulative effects over sequential tasks lead to progressive performance degradation on early-learned knowledge.

\subsection{Problem Formulation}


We consider a sequence of tasks, denoted as $\mathcal{T} = \{\mathcal{T}^t\}_{t=1}^T$, where $T$ is the total number of tasks. For each task $t$, the dataset is given by $\mathcal{T}^t = \{(\mathbf{x}_i^t, \mathbf{y}_i^t)\}_{i=1}^{N^t}$, where each sample consists of an input RGB image $\mathbf{x}_i^t \in \mathbb{R}^{H \times W \times 3}$ and its corresponding pixel-wise label map $\mathbf{y}_i^t \in \mathcal{Y}^{H \times W}$. Each pixel in the mask (label map) is assigned a semantic class from the label space $\mathcal{Y}^t$, and $N^t$ denotes the number of image-mask pairs in task $t$. The label space $\mathcal{Y}^t$ contains $C^t$ new classes that are unique to the current task. For example, in Mars terrain segmentation, Task 0 provides $N^0$ training samples, each consisting of an RGB image $\mathbf{x}_i^0 \in \mathbb{R}^{H \times W \times 3}$ (a 3-channel color image with height $H$ and width $W$ pixels) and its ground-truth label map $\mathbf{y}_i^0 \in \mathcal{Y}^{H \times W}$ (a same-sized matrix where each spatial location $(h,w)$ stores the terrain category of that pixel). Task 0 learns $C^0$ terrain classes $\mathcal{Y}^0=\{\text{Soil, Sand, Bedrock}\}$, meaning each pixel is classified as one of these three types. Task 1 then introduces $N^1$ images with $C^1$ classes $\mathcal{Y}^1=\{\text{Gravel, Rocks}\}$. Although Task 1 provides annotations only for its new classes, the model must be able to segment each pixel into one of the $C^0 + C^1$ classes learned so far.

In the context of federated continual learning (FCL), a key constraint is that the label spaces for different tasks are strictly non-overlapping, that is, $\mathcal{Y}^t \cap (\cup_{j=1}^{t-1}\mathcal{Y}^j) = \emptyset$. This formulation reflects a practical reality of the mission lifecycle: as rovers progressively enter new geological regions, they encounter terrain types that were absent from earlier mission stages, and the complete set of terrain categories cannot be predetermined at deployment time. As a result, each new task introduces $C^t$ classes that are entirely different from the $C^o = \sum_{i=1}^{t-1}C^i$ classes in previous tasks. This requires the model to incrementally expand its segmentation capability to accommodate new classes, while maintaining good performance on previously learned classes.

Each rover $\mathcal{C}_k$ receives a local subset $\mathcal{T}_k^t \subset \mathcal{T}^t$ consisting of $N_k^t$ samples drawn from its regional distribution $\mathcal{P}_k$. The local label space $\mathcal{Y}_k^t \subseteq \mathcal{Y}^t$ contains $C_k^t$ classes, and the global label space is formed through the union: $\mathcal{Y}^t = \cup_{k=1}^K \mathcal{Y}_k^t$. The data distributions $\{\mathcal{P}_k\}_{k=1}^K$ are Non-IID due to regional heterogeneity. To mitigate catastrophic forgetting, each rover maintains an exemplar memory buffer $\mathcal{M}_k$ that stores representative image-mask pairs from previously learned classes. The memory capacity $|\mathcal{M}_k|$ is fixed due to onboard storage constraints. As the number of learned classes $C^o$ grows across tasks, the per-class sample allocation $|\mathcal{M}_k|/C^o$ diminishes. 

The learning process operates over iterative communication rounds. Within each task $\mathcal{T}^t$, the training proceeds for $R$ communication rounds indexed by $r = 1, \ldots, R$. At round $r$, the server $\mathcal{S}$ distributes the current global model $\Theta^{r,t}$ to a selected subset of rovers $\hat{\mathcal{C}}^r \subseteq \{\mathcal{C}_1, \ldots, \mathcal{C}_K\}$. Each selected rover $\mathcal{C}_k \in \hat{\mathcal{C}}^r$ then performs local optimization on the combined dataset $\mathcal{T}_k^t \cup \mathcal{M}_k$, seeking to minimize:
\begin{equation}
\begin{aligned}
\min_{\Theta_k^{r,t}} \;& 
\mathbb{E}_{(\mathbf{x},\mathbf{y}) \sim \mathcal{T}_k^t}
\bigl[\ell(f_{\Theta_k}(\mathbf{x}), \mathbf{y})\bigr] \\
&+ \lambda \mathbb{E}_{(\mathbf{x},\mathbf{y}) \sim \mathcal{M}_k}
\bigl[\ell(f_{\Theta_k}(\mathbf{x}), \mathbf{y})\bigr],
\end{aligned}
\end{equation}
where $f_{\Theta_k}$ denotes the segmentation model parameterized by $\Theta_k^{r,t}$, $\ell$ is the pixel-wise cross-entropy loss, and $\lambda$ balances learning on new task data against rehearsal on memorized samples. Upon receiving the locally updated models $\{\Theta_k^{r,t}\}_{k \in \hat{\mathcal{C}}^r}$ from all selected rovers, the server performs federated aggregation to obtain the global model for the next round: $\Theta^{r+1,t} = \frac{1}{|\hat{\mathcal{C}}^r|} \sum_{k \in \hat{\mathcal{C}}^r} \Theta_k^{r,t}$.

The main \textbf{goal} of our FCL framework is to maximize segmentation performance in the entire mission lifecycle. To measure this, we use the average mean Intersection over Union (mIoU) calculated over all classes after the final task $T$ is completed, since mIoU gives how well the model maintains segmentation accuracy for all classes seen during the mission.

The changes in mIoU over time show two main types of degradation, each linked to a different phase of the mission lifecycle. During training within a task, the rover faces immediate conflicts between learning from new task data $\mathcal{T}_k^t$ and keeping the knowledge already stored in the local model $\mathcal{M}_k$. Each local update must balance learning new classes (plasticity) with keeping previous knowledge (stability). As a result, sudden drops in mIoU for earlier tasks right after training on $\mathcal{T}^t$ can be observed.

Over the long term, as the model progresses through the full sequence of tasks, a slower but ongoing degradation appears. Even if strategies during training help reduce immediate forgetting, repeated rounds of local updates and global aggregation, especially under strong Non-IID conditions, cause the representations of early-learned classes to drift. This problem becomes more serious as the memory buffer $\mathcal{M}_k$ covers less of each class when more classes are learned. Because of this, we observe a steady drop in mIoU for tasks learned earlier, which is due to the repeated accumulation of aggregation errors over the course of the mission.

\section{Proposed Method}
\label{sec:proposed_method}

\subsection{Initial Test and Observation}
\label{motivation}

\begin{figure}[t!]
\centering
\begin{minipage}[t]{0.515\columnwidth}
  \centering
  \includegraphics[width=\textwidth,height=9cm,keepaspectratio]{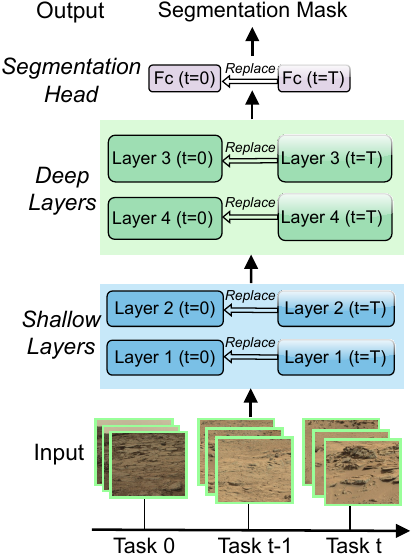}
  
  \vspace{0.2cm}
  {\small (a) Layer replacement}
\end{minipage}%
\hspace{0.02\columnwidth}%
\begin{minipage}[t]{0.465\columnwidth}
  \centering
  \includegraphics[width=\textwidth,height=9cm,keepaspectratio]{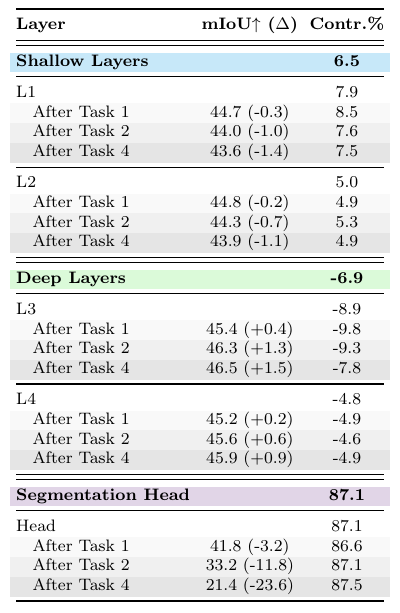}
  
  \vspace{0.2cm}  
  {\small (b) Forgetting analysis}
\end{minipage}
\caption{Layer-wise sensitivity analysis. (a) Controlled replacement experiment where baseline layers from $\Theta^{0}$ are individually replaced with trained layers from $\Theta^{t}$. (b) Quantitative forgetting contribution across different layer groups.}
\label{fig:layer_analysis}
\end{figure}

To understand catastrophic forgetting in FCL, we conduct a controlled layer replacement experiment to isolate the impact of each network layer on performance degradation. The experimental procedure is designed as follows. First, we train a baseline model $\Theta^{0}$ on Task 0. Next, we train a model $\Theta^{t}$ on Task $t$ using standard FCL. We then replace each layer of $\Theta^{0}$ individually with the corresponding layer from $\Theta^{t}$, keeping all other layers fixed. The change in mIoU on Task 0 after each replacement quantifies the specific contribution of that layer to forgetting.

We base this study on MarsScapes~\cite{Liu2023AA}, the first panoramic Mars terrain dataset collected by the Curiosity rover at Gale Crater, under extreme Non-IID conditions, using a Dirichlet distribution with $\beta=0.1$ across $K=30$ rovers. Task 0 includes classes 0–4 (Soil, Sand, Bedrock, Gravel, Rocks) and achieves an initial mIoU of 45.0\%. Task $t$ introduces new terrain classes, while requiring the retention of Task 0 knowledge through the exemplar memory $\mathcal{M}_k$.

The results in Fig.~\ref{fig:layer_analysis} indicate significant variation in the forgetting performance over different network layers. Each layer group's contribution is reported as its signed $\Delta$mIoU normalized by $\sum_l |\Delta\text{mIoU}_l|$ (unsigned shares sum to $\approx\!100\%$ up to rounding), which is an empirical counterpart of the $\kappa_l$ in Lemma~\ref{lemma:decomposition} rather than a numerical identity. Specifically, the segmentation head contributes $87.1\%$ of the total forgetting magnitude, with mIoU decreasing from $45.0\%$ to $21.4\%$. In contrast, replacing deeper backbone layers yields negative forgetting of $-6.9\%$, improving Task~0 performance from $45.0\%$ to $46.5\%$. Shallow layers contribute minimally ($6.5\%$), causing a $1.4$ pp decrease in mIoU.

Note that existing \emph{rehearsal-based} methods~\cite{Rebuffi2017CVPR, He2025TMC}, which mitigate forgetting by maintaining memory buffers $\mathcal{M}_k$ and replaying samples from earlier tasks during new task training, apply the same level of protection to all layers. However, our test results demonstrate that this approach is not optimal, because applying uniform rehearsal constraints presents a dilemma: strong constraints limit beneficial adaptation in the backbone, while weak constraints allow significant forgetting in the segmentation head. This motivates us to design rehearsal mechanisms that enforce stability where necessary and maintain plasticity where it is beneficial.


\subsection{The Proposed Framework}

\begin{figure*}[t]
\centering
\includegraphics[width=\textwidth]{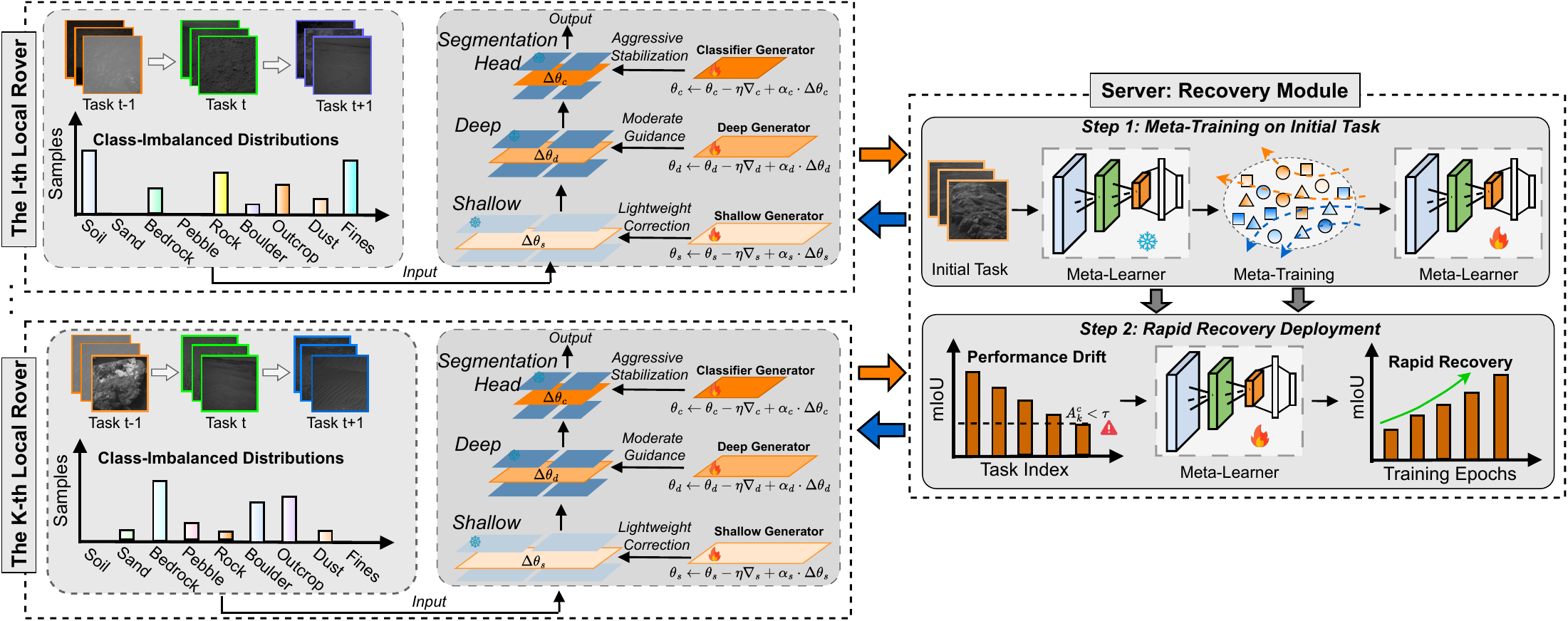}
\caption{Overview of the proposed dual-timescale framework. 
\textbf{Left:} Layer-Selective Rehearsal (LSR) operates during local 
training on rovers, applying stratified corrections through learned 
generators ($\phi_s$, $\phi_d$, $\phi_c$) that produce adaptive updates 
for shallow, deep, and segmentation head layers with heterogeneous protection 
strengths ($\alpha_s < \alpha_d < \alpha_c$). \textbf{Right:} Rapid 
Knowledge Recovery (RKR) is meta-trained on the initial task and deployed 
when long-term degradation is detected, efficiently restoring segmentation 
head performance through learned recovery patterns ($\psi$).}
\label{fig:framework}
\end{figure*}


The proposed framework is shown in Fig.~\ref{fig:framework}. The framework consists of two components for two different time scales: \textit{Layer-Selective Rehearsal (LSR)} operates during local training by applying stratified corrections via learned generators for shallow, deep, and segmentation head layers, while \textit{Rapid Knowledge Recovery (RKR)} provides server-side meta-learned recovery when long-term degradation is detected. At the training-time scale, LSR applies heterogeneous protection strategies through learned generators that decouple plasticity and stability requirements across layer groups. At the long-term scale, RKR provides efficient repair mechanisms that restore degraded segmentation head states through learned correction patterns.

LSR contributes a stratified, online-learned rehearsal strategy for the training-time scale, together with a closed-form suboptimality bound for any uniform-$\alpha$ rehearsal (Theorem~1). RKR contributes a single-pass recovery function meta-trained once on Task~0 and applied without further retraining at subsequent tasks, together with a sample-complexity statement (Corollary~1) that compares its fine-tuning cost against retraining. The design departs from prior layer-selective approaches: protection strengths are continuously stratified ($\alpha_s<\alpha_d<\alpha_c$) rather than binary, and the corrections are produced by learned generators conditioned on the current gradient and memory features, so the per-layer adjustment adapts as training proceeds rather than being fixed in advance. The two components are combined because Theorem~2 indicates that training-time prevention alone cannot bound the long-term degradation under Non-IID federated aggregation, so a recovery mechanism is required in addition.

\subsubsection{Layer-Selective Rehearsal (LSR)}\label{sec:lsr}


LSR applies layer-aware protection strategies during continual learning. Unlike existing rehearsal methods~\cite{Buzzega2020NeurIPS} that impose uniform constraints on all model parameters, LSR separates the plasticity–stability trade-off by allowing the backbone to evolve where beneficial while explicitly stabilizing the segmentation head. To realize this design, LSR employs three learned generator functions, ${\phi_s, \phi_d, \phi_c}$, which provide adaptive update corrections for shallow layers, deep layers, and the segmentation head, respectively. Each generator $\phi_l$, where $l \in \{s, d, c\}$, takes as input compact statistics of the current layer parameters and gradient together with memory-derived features $\mathbf{h}_m$, and outputs a correction term $\Delta\theta_l$ that modulates the standard gradient update. The resulting corrections differ in magnitude to reflect layer-specific sensitivity: shallow layers receive minimal intervention weighted by $\alpha_s$, deep layers receive moderate guidance weighted by $\alpha_d$, and the segmentation head receives strong stabilization weighted by $\alpha_c$, with $\alpha_s < \alpha_d < \alpha_c$. These generators are trained online alongside the model parameters, enabling LSR to adapt continuously as the semantic class space expands over sequential tasks.

Each generator $\phi_l$ operates on compact per-channel summaries of the current layer rather than on its full parameter vector. Let $z_l$ denote the concatenation of the per-channel mean and variance of $\theta_l$, and let $g_l$ denote the corresponding per-channel summary of the gradient $\nabla_l$. The generator input is $[z_l; g_l; \mathbf{h}_m]$, where $\mathbf{h}_m$ is obtained by globally average-pooling the backbone feature maps over the exemplars in $\mathcal{M}_k$. The generator output is a per-channel modulation factor that is broadcast to all parameters within the same channel to form $\Delta\theta_l$. This design keeps the generator size independent of the layer width and enables $\phi_l$ to be instantiated as compact MLPs.

The \textit{shallow layer generator} $\phi_s$ makes lightweight correction with minimal architectural complexity. Given the per-channel summaries $z_s, g_s \in \mathbb{R}^{2C_s}$ of the shallow-layer parameters and gradient (with $C_s$ the shallow-layer channel count) and the memory encoding $\mathbf{h}_m \in \mathbb{R}^{d_m}$, the generator is instantiated as a compact multi-layer perceptron that produces correction:
\begin{equation}\label{equ2}
\Delta\theta_s = \phi_s([z_s; g_s; \mathbf{h}_m]; \omega_s),
\end{equation}
where $[\cdot; \cdot; \cdot]$ denotes concatenation and $\omega_s$ represents learnable parameters. The lightweight design reflects the minimal forgetting contribution (6.5\%) while preserving plasticity for learning new task features. 

The \textit{deep layer generator} $\phi_d$ is designed to guide the beneficial evolution observed in sensitivity analysis (-6.9\%). It produces corrections that enhance rather than constrain adaptation:
\begin{equation}\label{equ3}
\Delta\theta_d = \phi_d([z_d; g_d; \mathbf{h}_m]; \omega_d),
\end{equation}
where $z_d, g_d \in \mathbb{R}^{2C_d}$ are the per-channel summaries of the deep-layer parameters and gradient. The moderately sized MLP learns to amplify update directions that are beneficial for both new task learning and old task preservation, allowing for natural feature evolution while memory-guided corrections prevent drift toward task-specific patterns. 

The \textit{segmentation head generator} $\phi_c$ implements stabilization to address the dominant forgetting source (87.1\%). Beyond parameters and gradients, it incorporates preserved knowledge including class prototypes $\mathbf{p} \in \mathbb{R}^{C^o \times d}$, the memory-pooled feature $\mathbf{h}_m$ (backbone-pooled features of the exemplars in $\mathcal{M}_k$, i.e., the same representation used in Eqs.~\eqref{equ2}--\eqref{equ3}), and geometric features $\mathbf{a}$ capturing inter-class angular relationships:
\begin{equation}\label{equ4}
\Delta\theta_c = \phi_c([z_c; g_c; \mathbf{p}; \mathbf{h}_m; \mathbf{a}]; \omega_c),
\end{equation}
where $z_c, g_c \in \mathbb{R}^{2C_c}$ are the per-channel summaries of the segmentation-head parameters and gradient. The architecture employs multiple specialized encoders for different input modalities, followed by a fusion layer that produces correction $\Delta\theta_c$. Operating on $\mathbf{h}_m$ rather than on raw image-mask pairs keeps the input dimensionality tractable for a compact MLP-based generator. The richer input context enables precise boundary preservation while accommodating new class integration. The parameter update for each layer group $l$ integrates standard gradient descent with learned correction weighted by $\alpha_l$:
\begin{equation}\label{equ5}
\theta_l \leftarrow \theta_l - \eta \nabla_l + \alpha_l \Delta\theta_l,
\end{equation}
where $\eta$ is the learning rate. The stratified weighting $\alpha_s < \alpha_d < \alpha_c$ ensures minimal intervention in shallow layers, moderate guidance in deep layers, and aggressive stabilization in the segmentation head, directly reflecting their respective forgetting contributions.

The generators must be trained online rather than pre-trained due to incremental evaluation characteristics of the FCL. After completing task $t$, models are evaluated on the cumulative test set spanning tasks 0 through $t$, requiring generators to adapt continuously to the expanding label space. Pre-trained generators with fixed capacity cannot accommodate the growing number of classes and evolving class relationships across $T$ sequential tasks. During local training on rover $\mathcal{C}_k$ at round $r$ of task $t$, we jointly optimize model parameters $\Theta_k^{r,t} = \{\theta_s, \theta_d, \theta_c\}$ and generator parameters $\Omega = \{\omega_s, \omega_d, \omega_c\}$. The selective rehearsal loss combines task learning with memory rehearsal:
\begin{equation}\label{equ6}
\begin{aligned}
\mathcal{L}_s \;=\;&
\mathbb{E}_{(\mathbf{x},\mathbf{y}) \sim \mathcal{T}_k^t}
\bigl[\ell(f_\Theta(\mathbf{x}), \mathbf{y})\bigr] \\
&+ \lambda \mathbb{E}_{(\mathbf{x},\mathbf{y}) \sim \mathcal{M}_k}
\bigl[\ell(f_\Theta(\mathbf{x}), \mathbf{y})\bigr],
\end{aligned}
\end{equation}
where $\ell$ denotes pixel-wise cross-entropy loss and $\lambda$ balances new task learning against old task preservation. The generators are updated to minimize forgetting while enabling model parameters to learn effectively.

Note in each communication round that, selected rovers receive global model $\Theta^{r,t}$ from server $\mathcal{S}$ and perform local training with LSR. The generators $\{\phi_s, \phi_d, \phi_c\}$ operate locally during gradient computation, giving corrections that guide parameter updates according to the stratified rule. After local training, rovers transmit updated model parameters to the server for aggregation: $\Theta^{r+1,t} = \frac{1}{|\hat{\mathcal{C}}^r|} \sum_{k \in \hat{\mathcal{C}}^r} \Theta_k^{r,t}$. Generator parameters $\Omega$ remain local and are not communicated, ensuring minimal overhead beyond standard federated learning.

\subsubsection{Rapid Knowledge Recovery (RKR)}\label{sec:rkr}


Although LSR mitigates forgetting at the training-time scale, long-term degradation over sequential tasks remains challenging. This is because, on the one hand, prior work~\cite{Wang2024NeurIPS,Khademi2024NeurIPS} has established that catastrophic forgetting in continual learning is inevitable due to the accumulation of incremental errors over extended task sequences, which are further augmented by repeated aggregation under Non-IID conditions. On the other hand, while fine-tuning on memory buffers can partially restore performance, the communication overhead and computational cost make frequent complete retraining impossible in resource-constrained rovers. To address these issues, we introduce RKR that learns to efficiently restore degraded segmentation head states by leveraging preserved knowledge for targeted corrections.

The recovery function $\psi$ operates exclusively on the segmentation head to address the dominant forgetting source. Given a degraded segmentation head state $\theta_c^* \in \mathbb{R}^{d_c}$ after task $t$, the function generates correction $\Delta\theta_c$ by integrating multiple sources of preserved knowledge. The architecture employs specialized encoders for degraded parameters ($E_s$), class prototypes ($E_p$), memory exemplars ($E_m$), and geometric features ($E_g$), which are fused through attention-weighted aggregation. The complete forward pass is formulated as:
\begin{equation}\label{equ7}
\begin{aligned}
\Delta\theta_c &= \psi(\theta_c^*, \mathbf{p}, \mathcal{M}_k, \mathbf{a}; \xi) \\
&= \mathcal{F}\left(
    \sum_{i \in \{s,p,m,g\}}
    \frac{\exp(\mathbf{q}^\top E_i(\cdot))}{\sum_{j} \exp(\mathbf{q}^\top E_j(\cdot))}
    \cdot E_i(\cdot)
\right),
\end{aligned}
\end{equation}
where $E_s(\theta_c^*)$, $E_p(\mathbf{p})$, $E_m(\mathcal{M}_k)$, and $E_g(\mathbf{a})$ produce encoded representations, $\mathbf{q}$ is a learnable query vector computing attention weights via softmax, and $\mathcal{F}$ denotes the fusion decoder network parameterized by $\xi$. The recovered segmentation head is obtained as $\theta_c^+ = \theta_c^* + \Delta\theta_c$, followed by brief fine-tuning on $\mathcal{M}_k$ to stabilize restoration.

The recovery function is trained on task 0 using episodic procedures inspired by meta-learning principles~\cite{Wang2022CVPR}. Each episode simulates degradation by partitioning task 0's label space $\mathcal{Y}^0$ into training subset $\mathcal{Y}^-$ and held-out subset $\mathcal{Y}^+ = \mathcal{Y}^0 \setminus \mathcal{Y}^-$. A degraded segmentation head $\theta_c^*$ is constructed by optimizing only on $\mathcal{Y}^-$, inducing artificial forgetting on $\mathcal{Y}^+$. The recovery parameters $\xi$ are then optimized through a bi-level objective:
\begin{align}\label{equ8}
\min_\xi \mathbb{E}_{\mathcal{Y}^-, \mathcal{Y}^+} &\left[ \frac{1}{|\mathcal{D}_{\mathcal{Y}^+}|} \sum_{(\mathbf{x},\mathbf{y}) \in \mathcal{D}_{\mathcal{Y}^+}} \ell(f_{\theta_c^* + \psi(\cdot; \xi)}(\mathbf{x}), \mathbf{y}) \right],
\end{align}
where $\theta_c^* = \arg\min_{\theta_c} \frac{1}{|\mathcal{D}_{\mathcal{Y}^-}|} \sum_{(\mathbf{x},\mathbf{y}) \in \mathcal{D}_{\mathcal{Y}^-}} \ell(f_{\theta_c}(\mathbf{x}), \mathbf{y})$, $\mathcal{D}_{\mathcal{Y}^-}$ and $\mathcal{D}_{\mathcal{Y}^+}$ denote data distributions over respective label subsets, and $\ell$ is the pixel-wise cross-entropy loss. 

The above episodic training learns general recovery patterns rather than class-specific corrections, which enables direct transfer to tasks $t > 0$ without retraining $\psi$. During deployment phase, each rover monitors cumulative mIoU $\mathcal{I}_k^c$ after completing task $t$. When performance on early tasks drops below threshold $\tau$, the rover invokes recovery by generating $\Delta\theta_c$ using current prototypes $\mathbf{p}$, memory buffer $\mathcal{M}_k$, and geometric features $\mathbf{a}$, then applies the correction $\theta_c^+ = \theta_c^* + \Delta\theta_c$ followed by brief fine-tuning. The recovery function parameters $\xi$ remain constant in all tasks that requires no additional training or communication beyond the initial Task 0 meta-learning phase, maintaining the minimal overhead.

\subsection{Training Pipeline}

\begin{algorithm}[t!]
\caption{Training Lifecycle in FCL (\colorbox{rgb:red!2,65;green!30,60;blue!20,125}{LSR Training}, \colorbox{rgb:red!2,65;green!30,90;blue!20,125}{RKR Triggering}, and \colorbox{rgb:red!30,155;green!20,20;blue!20,30}{Aggregation})}\label{alg:training}
\textbf{Input:} $K$ rovers $\{\mathcal{C}_k\}_{k=1}^K$, $T$ tasks, $R$ rounds per task, threshold $\tau$
\begin{algorithmic}[1]
\STATE \textbf{Task 0 Initialization:}
\STATE Initialize $\Theta^{0,0}$, memory $\{\mathcal{M}_k\}$, generators $\{\phi_s, \phi_d, \phi_c\}$
\STATE Train recovery function $\psi$ via episodic meta-learning
\FOR{$t = 1$ to $T$}
    \FOR{$r = 1$ to $R$}
        \colorbox{rgb:red!2,65;green!30,60;blue!20,125}{
        \parbox{0.82\columnwidth}{
        \STATE \gray{$\triangleright$ \textit{Rovers perform local training with LSR:}}
        \FOR{each selected rover $\mathcal{C}_k \in \hat{\mathcal{C}}^r$}
            \STATE Receive $\Theta^{r,t}$; Initialize $\Theta_k^{r,t} = \Theta^{r,t}$;
            \FOR{each batch $(\mathbf{x}, \mathbf{y}) \sim \mathcal{T}_k^t \cup \mathcal{M}_k$}
                \STATE Compute $\nabla_s, \nabla_d, \nabla_c$ from $\mathcal{L}_s$ (Eq.~\eqref{equ6});
                \STATE Generate corrections (Eqs.~\eqref{equ2}-\eqref{equ4}); 
                \STATE Update layers (Eq.~\eqref{equ5});
                \STATE Update generators $\Omega$;
            \ENDFOR
            \STATE Upload $\Theta_k^{r,t}$ to server;
        \ENDFOR
        }}
            
        \colorbox{rgb:red!30,155;green!20,20;blue!20,30}{
        \parbox{0.82\columnwidth}{
        \STATE \gray{$\triangleright$ \textit{Server performs aggregation:}}
        \STATE $\Theta^{r+1,t} \leftarrow \frac{1}{|\hat{\mathcal{C}}^r|} \sum_{k \in \hat{\mathcal{C}}^r} \Theta_k^{r,t}$; Distribute to rovers;
        }}
    \ENDFOR
    
    \colorbox{rgb:red!2,65;green!30,90;blue!20,125}{
    \parbox{0.82\columnwidth}{
    \STATE \gray{$\triangleright$ \textit{Long-term recovery triggering:}}
    \FOR{each rover $\mathcal{C}_k$}
        \STATE Evaluate cumulative mIoU $\mathcal{I}_k^c$;
        \IF{$\mathcal{I}_k^c < \tau$}
            \STATE Generate $\Delta\theta_c$ via $\psi$ (Eq.~\eqref{equ7}); 
            \STATE Apply $\theta_c^+ \leftarrow \theta_c^* + \Delta\theta_c$;
            \STATE Fine-tune on $\mathcal{M}_k$ for few epochs;
        \ENDIF
    \ENDFOR
    }}
\ENDFOR
\RETURN $\Theta^{R,T}$
\end{algorithmic}
\end{algorithm}

The training pipeline of our framework is presented in \textbf{Algorithm~\ref{alg:training}}. Task 0 serves as the initialization phase where all rovers train the initial model on their local data and construct memory buffers $\mathcal{M}_k$ following iCaRL~\cite{Rebuffi2017CVPR}. Note that the recovery function $\psi$ is trained during Task 0 via the episodic meta-learning procedure described in Section~\ref{sec:rkr}, learning transferable recovery patterns that will be applied to all subsequent tasks without retraining. The learned generators $\{\phi_s, \phi_d, \phi_c\}$ are initialized and will be updated online throughout the training.

For each task $t \geq 1$, the training proceeds through $R$ communication rounds. At round $r$, the server $\mathcal{S}$ distributes the global model $\Theta^{r,t}$ to selected rovers $\hat{\mathcal{C}}^r$ (blue lines 6-16 of \textbf{Algorithm~\ref{alg:training}}). Each selected rover $\mathcal{C}_k$ performs local training by sampling batches from $\mathcal{T}_k^t \cup \mathcal{M}_k$ and applying LSR. For each batch, the rover computes task gradients and generates layer-wise corrections via the three generators following Equations~\eqref{equ2}-\eqref{equ4}, then updates parameters according to the stratified rule in Equation~\eqref{equ5}. The generators $\Omega = \{\omega_s, \omega_d, \omega_c\}$ are jointly optimized with model parameters to minimize the selective rehearsal loss in Eq.~\eqref{equ6}. After local training, rovers upload their updated models to the server, which performs federated aggregation $\Theta^{r+1,t} = \frac{1}{|\hat{\mathcal{C}}^r|} \sum_{k \in \hat{\mathcal{C}}^r} \Theta_k^{r,t}$ (orange lines 17-18 of \textbf{Algorithm~\ref{alg:training}}). Generator parameters remain local and are not communicated.

After completing all $R$ rounds for task $t$, each rover evaluates its cumulative mIoU on tasks $\{0, \ldots, t\}$ (see green lines 20-28 of \textbf{Algorithm~\ref{alg:training}}). When performance on early tasks drops below threshold $\tau$, RKR is invoked. The rover generates correction $\Delta\theta_c$ using the pre-trained recovery function $\psi$ via Equation~\eqref{equ7}, applies the correction $\theta_c^+ = \theta_c^* + \Delta\theta_c$, and performs brief fine-tuning on memory buffer $\mathcal{M}_k$. Since the recovery function parameters $\xi$ remain fixed after Task 0 training, the process incurs minimal computational overhead and requires no additional communication with the server.
\section{Theoretical Analysis}
\label{sec:theoretical_analysis}

In this section, we establish theoretical foundations that support our dual-timescale framework. We first formalize the forgetting metric and state the necessary assumptions. We then characterize training-time forgetting through layer-wise decomposition and prove the suboptimality of uniform rehearsal strategies. Finally, we analyze long-term degradation inevitability and derive sample complexity bounds for recovery mechanisms.

The theoretical analysis in this section is intended to provide qualitative design guidance rather than tight performance bounds. The three results that go beyond standard machinery are: Lemma~\ref{lemma:decomposition}, a layer-wise decomposition of forgetting that provides the structural basis for stratified protection; Theorem~\ref{thm:uniform_suboptimal}, a closed-form suboptimality bound for any uniform-$\alpha$ rehearsal; and Theorem~\ref{thm:longterm}, a long-term degradation lower bound that couples federated aggregation drift with memory-buffer-coverage dilution. The remaining steps (smoothness-based bounds, gradient-variance decomposition) use standard tools and serve as building blocks for these three results.

\begin{definition}[Task-Level Forgetting]\label{def:forgetting}
Let $\mathcal{P}^t$ denote the global test distribution for task $t$. 
For task $t' < t$, the forgetting of task $t'$ after training on 
task $t$ is defined as:
\begin{equation}
\begin{aligned}
\Delta^{t,t'} :=\;&
\mathbb{E}_{(\mathbf{x},\mathbf{y}) \sim \mathcal{P}^{t'}}
\bigl[\ell(f_{\Theta^t}(\mathbf{x}), \mathbf{y})\bigr] \\
&- \mathbb{E}_{(\mathbf{x},\mathbf{y}) \sim \mathcal{P}^{t'}}
\bigl[\ell(f_{\Theta^{t'}}(\mathbf{x}), \mathbf{y})\bigr].
\end{aligned}
\end{equation}
The cumulative forgetting at task $t$ is:
\begin{equation}
\Delta^t := \frac{1}{t} \sum_{t'=0}^{t-1} \Delta^{t,t'}.
\end{equation}
\end{definition}

\noindent We now present the assumptions required for our theoretical results.

\begin{assumption}[Layer-wise Smoothness]\label{asm:smoothness}
For each layer $l \in \{s, d, c\}$, the loss function is $L_l$-smooth with respect to layer parameters, i.e., for any $\theta_l, \tilde{\theta}_l$:
\begin{equation}
\begin{aligned}
&\|\nabla_{\theta_l} \ell(f_\Theta(\mathbf{x}), \mathbf{y}) 
- \nabla_{\tilde{\theta}_l} \ell(f_{\tilde{\Theta}}(\mathbf{x}), \mathbf{y})\| \\
&\qquad\qquad\qquad\leq L_l \|\theta_l - \tilde{\theta}_l\|,
\end{aligned}
\end{equation}
where $L_s, L_d, L_c$ are layer-specific smoothness parameters.
\end{assumption}

\begin{assumption}[Bounded Gradient Variance]\label{asm:variance}
For any model parameters $\Theta$ and rover $k$, the stochastic gradient has bounded variance:
\begin{equation}
\begin{aligned}
&\mathbb{E}_{(\mathbf{x},\mathbf{y}) \sim \mathcal{P}_k}
\Big[\|\nabla_\Theta \ell(f_\Theta(\mathbf{x}), \mathbf{y}) \\
&\qquad\qquad - \nabla_\Theta \mathbb{E}[\ell]\|^2\Big] 
\leq \sigma^2.
\end{aligned}
\end{equation}
\end{assumption}

\begin{assumption}[Distribution Heterogeneity]\label{asm:heterogeneity}
The Non-IID degree is characterized by heterogeneity parameter $\gamma^2$:
\begin{equation}
\begin{aligned}
&\frac{1}{K} \sum_{k=1}^K 
\Big\|\nabla_\Theta \mathbb{E}_{(\mathbf{x},\mathbf{y}) \sim \mathcal{P}_k}[\ell] \\
&\qquad - \nabla_\Theta \mathbb{E}_{(\mathbf{x},\mathbf{y}) \sim \bar{\mathcal{P}}}[\ell]\Big\|^2 
\leq \gamma^2,
\end{aligned}
\end{equation}
where $\bar{\mathcal{P}} = \frac{1}{K}\sum_{k=1}^K \mathcal{P}_k$ is the average distribution.
\end{assumption}

\begin{assumption}[Memory Buffer Constraint]\label{asm:memory}
Each rover maintains a fixed-capacity memory buffer $|\mathcal{M}_k| = M$. As tasks accumulate, the per-class coverage diminishes:
\begin{equation}
\frac{|\mathcal{M}_k|}{C^o} = \frac{M}{\sum_{i=0}^{t-1}C^i}
\to 0 \quad \text{as } t \to \infty.
\end{equation}
\end{assumption}

We acknowledge that the layer-wise smoothness assumption and the bounded gradient variance assumption are idealized conditions when applied to deep segmentation networks trained under strong heterogeneity. The theoretical results in this section should therefore be read as directional insight about the design: they indicate which failure modes are structural (layer-wise heterogeneous forgetting, long-term degradation under Non-IID aggregation) rather than artifacts of a particular training run. The experiments in Section~\ref{sec:experiments} are intended to provide empirical validation of these qualitative predictions.

\subsection{Training-Time Forgetting Characterization}

We now characterize how catastrophic forgetting manifests heterogeneously across network layers and establish the theoretical foundation for layer-selective rehearsal.

\begin{lemma}[Layer-wise Forgetting Decomposition]\label{lemma:decomposition}
Let $\Theta^{t'} = \{\theta_s^{t'}, \theta_d^{t'}, \theta_c^{t'}\}$ and $\Theta^t = \{\theta_s^t, \theta_d^t, \theta_c^t\}$ be model states after tasks $t'$ and $t > t'$, respectively. Under Assumption~\ref{asm:smoothness}, there exist layer-wise contribution coefficients $\kappa_s, \kappa_d, \kappa_c \geq 0$ with $\kappa_s + \kappa_d + \kappa_c = 1$ such that:
\begin{equation}
\begin{aligned}
\Delta^{t,t'} ={}&
\sum_{l \in \{s,d,c\}} \kappa_l \,
\mathbb{E}_{\mathcal{P}^{t'}}
\Big[
\ell\big(f_{\Theta^{t',l \leftarrow t}}(\mathbf{x}), \mathbf{y}\big)
\\
&\quad
-
\ell\big(f_{\Theta^{t'}}(\mathbf{x}), \mathbf{y}\big)
\Big]
+ \mathcal{O}\!\left(\|\Theta^t - \Theta^{t'}\|^2\right),
\end{aligned}
\end{equation}
where $\Theta^{t',l \leftarrow t} := 
\Theta^{t'} \setminus \{\theta_l^{t'}\} \cup \{\theta_l^t\}$.
\end{lemma}

\begin{proof}
By Taylor expansion of the loss function around $\Theta^{t'}$:
\begin{equation}
\begin{aligned}
\ell(f_{\Theta^t}) - \ell(f_{\Theta^{t'}}) 
&\approx \sum_{l \in \{s,d,c\}} \langle \nabla_{\theta_l}\ell, \theta_l^t - \theta_l^{t'} \rangle \\
&\quad + \frac{1}{2}\sum_{l} (\theta_l^t - \theta_l^{t'})^\top 
\nabla^2_{\theta_l}\ell (\theta_l^t - \theta_l^{t'}).
\end{aligned}
\end{equation}
Define the layer-wise contribution coefficient as:
\begin{equation}
\kappa_l := \frac{\|\nabla_{\theta_l}\ell\| \cdot \|\theta_l^t - \theta_l^{t'}\|}{\sum_{l' \in \{s,d,c\}}\|\nabla_{\theta_{l'}}\ell\| \cdot \|\theta_{l'}^t - \theta_{l'}^{t'}\|}.
\end{equation}
By construction, $\sum_l \kappa_l = 1$. The first-order term captures the impact of replacing individual layers, which corresponds to the controlled layer replacement experiment. The second-order terms are bounded by Assumption~\ref{asm:smoothness}: for each layer $l$, we have $\|(\theta_l^t - \theta_l^{t'})^\top \nabla^2_{\theta_l}\ell (\theta_l^t - \theta_l^{t'})\| \leq L_l\|\theta_l^t - \theta_l^{t'}\|^2$, yielding the stated $\mathcal{O}(\|\Theta^t - \Theta^{t'}\|^2)$ residual.
\end{proof}

\begin{lemma}[Gradient Conflict Characterization]\label{lemma:gradient_conflict}
During training on task $t$, the gradient conflict between new task data and old task memory exhibits layer-wise heterogeneity. For classifier parameters:
\begin{equation}
\begin{aligned}
&\mathbb{E}\Big[\|\nabla^n_c - \nabla^o_c\|^2\Big] \\
&\geq \frac{L_c^2}{4} \cdot \Big(\frac{C^t}{C^o}\Big)^2 
\cdot \Big\|\mathbb{E}_{\mathbf{x} \sim \mathcal{P}^o}[\phi(\mathbf{x})] \\
&\qquad\qquad\qquad - \mathbb{E}_{\mathbf{x} \sim \mathcal{P}^n}[\phi(\mathbf{x})]\Big\|^2,
\end{aligned}
\end{equation}
where $\nabla^n_c$ and $\nabla^o_c$ denote gradients on new and old task data, $\mathcal{P}^n$ and $\mathcal{P}^o$ denote new and old data distributions. For deep layer parameters:
\begin{equation}
\mathbb{E}\Big[\|\nabla^n_d - \nabla^o_d\|^2\Big] 
\leq \frac{L_d^2}{d_d} \cdot 
\mathbb{E}\Big[\|\mathbf{x}^n - \mathbf{x}^o\|^2\Big],
\end{equation}
where $\phi(\mathbf{x})$ denotes backbone features, $d_d$ is the deep layer dimension, and $\mathbf{x}^n, \mathbf{x}^o$ are samples from new and old distributions.
\end{lemma}

\begin{proof}
For the classifier with weight matrix $\mathbf{W}_c \in \mathbb{R}^{(C^o+C^t) \times d}$, the gradient on new task data primarily updates rows corresponding to new classes. Let $\mathbf{e}_i$ denote the $i$-th standard basis vector. The new task gradient can be approximated as:
\begin{equation}
\nabla^n_c \approx \frac{1}{C^t}\sum_{i=C^o+1}^{C^o+C^t} 
\mathbf{e}_i \otimes \mathbb{E}_{\mathbf{x} \sim \mathcal{P}^n}[\phi(\mathbf{x})],
\end{equation}
while the old task gradient maintains existing class boundaries:
\begin{equation}
\nabla^o_c \approx \frac{1}{C^o}\sum_{i=1}^{C^o} 
\mathbf{e}_i \otimes \mathbb{E}_{\mathbf{x} \sim \mathcal{P}^o}[\phi(\mathbf{x})].
\end{equation}
Since $\{\mathbf{e}_i\}$ are orthogonal and the two approximations populate disjoint row sets (old-class vs.\ new-class), the squared-norm difference decomposes into two component squared norms. Applying $L_c$-Lipschitz continuity with the row-normalization factors $1/C^o$ and $1/C^t$ yields the stated bound in terms of the feature-mean shift $\|\mathbb{E}_{\mathbf{x}\sim\mathcal{P}^o}[\phi(\mathbf{x})] - \mathbb{E}_{\mathbf{x}\sim\mathcal{P}^n}[\phi(\mathbf{x})]\|$ directly.

For deep layers, feature evolution is continuous across tasks. The gradient difference is bounded by:
\begin{equation}
\|\nabla^n_d - \nabla^o_d\| 
\leq L_d \cdot \frac{1}{\sqrt{d_d}} \cdot \|\mathbb{E}[\mathbf{x}^n] - \mathbb{E}[\mathbf{x}^o]\|.
\end{equation}
The $1/\sqrt{d_d}$ factor arises from averaging over $d_d$ feature dimensions. Taking expectations and applying Cauchy-Schwarz inequality yields the stated upper bound.
\end{proof}

\begin{theorem}[Suboptimality of Uniform Rehearsal]\label{thm:uniform_suboptimal}
Let $\mathcal{R}_\alpha$ apply a uniform weight $\alpha$ to all layers, and let $\mathcal{R}_L$ apply stratified weights $\{\alpha_s, \alpha_d, \alpha_c\}$ with $\alpha_s \leq \alpha_d \leq \alpha_c$. We compare the two at matched average regularization, writing $\alpha_c = \alpha(1+\delta)$ and $\alpha_d = \alpha(1-\delta')$ with $\kappa_c \delta \approx \kappa_d \delta'$ (LSR redistributes protection across layers without imposing a globally stronger total). Under Assumptions~\ref{asm:smoothness} and~\ref{asm:variance}, if $\kappa_c > \kappa_d$, then for any task sequence of length $T$:
\begin{equation}
\begin{aligned}
\mathbb{E}[\Delta^T_L] 
&\leq \Big(1 - \frac{(\kappa_c - \kappa_d)(\alpha_c - \alpha_d)}{4\alpha}\Big) \\
&\quad \cdot \mathbb{E}[\Delta^T_u],
\end{aligned}
\end{equation}
where $\Delta^T_u$ and $\Delta^T_L$ denote cumulative forgetting under uniform and LSR strategies. Moreover, when $\kappa_c - \kappa_d = \Omega(1)$ and $\alpha_c - \alpha_d = \Omega(\alpha)$, LSR achieves a constant factor improvement over uniform rehearsal.
\end{theorem}

\begin{proof}
Let $\Delta_l^t$ denote the instantaneous forgetting contribution 
of layer $l$ at task $t$. For uniform rehearsal with weight $\alpha$ 
applied to all layers, the cumulative forgetting decomposes as:
\begin{equation}
\Delta^T_u = \sum_{l \in \{s,d,c\}} \kappa_l 
\cdot \sum_{t=1}^T (1-\alpha)^{T-t} \Delta_l^t,
\end{equation}
where $\Delta_l^t$ denotes the instantaneous forgetting contribution of layer $l$ at task $t$. For LSR with stratified weights, we have:
\begin{equation}
\begin{aligned}
\Delta^T_L 
&= \kappa_c \sum_{t=1}^T (1-\alpha_c)^{T-t} \Delta_c^t \\
&\quad + \kappa_d \sum_{t=1}^T (1-\alpha_d)^{T-t} \Delta_d^t \\
&\quad + \kappa_s \sum_{t=1}^T (1-\alpha_s)^{T-t} \Delta_s^t.
\end{aligned}
\end{equation}
Since $\alpha_c > \alpha > \alpha_d$, we have $(1-\alpha_c)^{T-t} < (1-\alpha)^{T-t} < (1-\alpha_d)^{T-t}$. The dominant term corresponds to the classifier due to $\kappa_c \gg \kappa_d$ (by Lemma~\ref{lemma:gradient_conflict}). 

To quantify the improvement, consider setting $\alpha_c = \alpha(1+\delta)$ and $\alpha_d = \alpha(1-\delta')$ where $\delta, \delta' > 0$. To maintain approximately equal total regularization, we require $\kappa_c \delta \approx \kappa_d \delta'$. The difference in cumulative forgetting becomes:
\begin{equation}
\begin{aligned}
&\Delta^T_u - \Delta^T_L \\
&\geq \kappa_c \sum_{t=1}^T \Big[(1-\alpha)^{T-t} - (1-\alpha_c)^{T-t}\Big] \Delta_c^t \\
&\quad - \kappa_d \sum_{t=1}^T \Big[(1-\alpha_d)^{T-t} - (1-\alpha)^{T-t}\Big] \Delta_d^t.
\end{aligned}
\end{equation}
Using the approximation $(1-\alpha)^{T-t} - (1-\alpha_c)^{T-t} \approx (T-t)\delta\alpha(1-\alpha)^{T-t-1}$ for small $\delta$, and noting that $\kappa_c \Delta_c^t \gg \kappa_d \Delta_d^t$ by Lemma~\ref{lemma:gradient_conflict}, the dominant term yields:
\begin{equation}
\Delta^T_u - \Delta^T_L \geq \frac{(\kappa_c - \kappa_d)\delta \alpha T(T-1)}{2} \bar{\Delta},
\end{equation}
where $\bar{\Delta}$ is the average instantaneous forgetting. Dividing by $\Delta^T_u \approx T\bar{\Delta}$ and choosing $\delta = (\alpha_c - \alpha_d)/(2\alpha)$ yields the stated bound.
\end{proof}

\begin{remark}\label{remark:lsr}
Theorem~\ref{thm:uniform_suboptimal} indicates that uniform rehearsal allocates regularization strength in proportion to layer count rather than in proportion to per-layer forgetting contribution. The closed-form factor is derived under the smoothness and bounded-variance assumptions and should be read as a directional statement about when stratification helps; the empirical ablation in Section~\ref{sec:experiments} provides the corresponding quantitative check.
\end{remark}

\subsection{Long-Term Degradation Analysis}

We now analyze the inevitable accumulation of degradation over extended task sequences and establish theoretical foundations for recovery mechanisms.

\begin{lemma}[Aggregation-Induced Drift]\label{lemma:drift}
Under Assumption~\ref{asm:heterogeneity}, for each federated aggregation round $\Theta^{r+1,t} = \frac{1}{|\hat{\mathcal{C}}^r|} \sum_{k \in \hat{\mathcal{C}}^r} \Theta_k^{r,t}$, the drift relative to the ideal global optimum $\hat{\Theta}$ satisfies:
\begin{equation}
\begin{aligned}
&\mathbb{E}\Big[\|\Theta^{r+1,t} - \hat{\Theta}\|^2\Big] \\
&\leq \Big(1 - \frac{\eta}{2L}\Big) 
\mathbb{E}\Big[\|\Theta^{r,t} - \hat{\Theta}\|^2\Big] 
+ \frac{2\eta^2\gamma^2}{|\hat{\mathcal{C}}^r|},
\end{aligned}
\end{equation}
where $\eta$ is the learning rate and $L = \max_l L_l$.
\end{lemma}

\begin{proof}
Each local update can be written as:
\begin{equation}
\Theta_k^{r,t} = \Theta^{r,t} 
- \eta \nabla_{\Theta_k} \mathbb{E}_{(\mathbf{x},\mathbf{y}) \sim \mathcal{P}_k}[\ell].
\end{equation}
After aggregation:
\begin{equation}
\begin{aligned}
\Theta^{r+1,t} - \hat{\Theta}
&= \Theta^{r,t} - \hat{\Theta} \\
&\quad - \eta \cdot \frac{1}{|\hat{\mathcal{C}}^r|}\sum_{k \in \hat{\mathcal{C}}^r} \nabla_{\Theta_k}\mathbb{E}[\ell].
\end{aligned}
\end{equation}
Taking norms and applying the smoothness condition from Assumption~\ref{asm:smoothness}:
\begin{equation}
\begin{aligned}
&\|\Theta^{r+1,t} - \hat{\Theta}\|^2 \\
&\leq \|\Theta^{r,t} - \hat{\Theta}\|^2 
- 2\eta \langle \bar{\nabla}, \Theta^{r,t} - \hat{\Theta} \rangle \\
&\quad + \eta^2 \|\bar{\nabla}\|^2,
\end{aligned}
\end{equation}
where $\bar{\nabla} = \frac{1}{|\hat{\mathcal{C}}^r|}\sum_{k \in \hat{\mathcal{C}}^r} \nabla_{\Theta_k}\mathbb{E}[\ell]$ is the aggregated gradient. By the smoothness assumption, $\langle \bar{\nabla}, \Theta^{r,t} - \hat{\Theta} \rangle \geq \|\Theta^{r,t} - \hat{\Theta}\|^2 / L$ when $\Theta^{r,t}$ is sufficiently close to $\hat{\Theta}$. Furthermore, by Assumption~\ref{asm:heterogeneity}:
\begin{equation}
\mathbb{E}\Big[\|\bar{\nabla} - \nabla_{\hat{\Theta}}\mathbb{E}[\ell]\|^2\Big] 
\leq \frac{\gamma^2}{|\hat{\mathcal{C}}^r|}.
\end{equation}
Taking expectations and noting that $\nabla_{\hat{\Theta}}\mathbb{E}[\ell] = 0$ at the optimum, we obtain:
\begin{equation}
\begin{aligned}
&\mathbb{E}\Big[\|\Theta^{r+1,t} - \hat{\Theta}\|^2\Big] \\
&\leq \Big(1 - \frac{\eta}{L}\Big)\mathbb{E}\Big[\|\Theta^{r,t} - \hat{\Theta}\|^2\Big] 
+ \frac{\eta^2\gamma^2}{|\hat{\mathcal{C}}^r|}.
\end{aligned}
\end{equation}
Applying the contraction factor $(1 - \eta/L) \leq (1 - \eta/(2L))$ when $\eta \leq 1/L$ completes the proof.
\end{proof}

\begin{theorem}[Inevitable Long-Term Degradation]\label{thm:longterm}
Consider a FCL setting with $T$ tasks, $R$ rounds per task, and $K$ rovers under Non-IID conditions characterized by $\gamma \geq \gamma_0$ for some constant $\gamma_0 > 0$. Assume the existence of a perfect training-time defense mechanism such that instantaneous forgetting $\Delta_i^t = 0$ for all $t \in [1, T]$. Nevertheless, the long-term cumulative degradation satisfies the lower bound:
\begin{equation}
\begin{aligned}
\mathbb{E}[\Delta^T] 
&\geq \Omega\Big(
\gamma^2 \cdot T \cdot R \cdot \Big(\frac{C^o}{M}\Big)^{1/2}
\Big) \\
&\geq \Omega\Big(
\gamma_0^2 \cdot T^{3/2} \cdot R
\Big),
\end{aligned}
\end{equation}
where $M = |\mathcal{M}_k|$ is the memory buffer size and $\Delta_i^t$ denotes instantaneous forgetting at task $t$. The second inequality uses $C^o = \Omega(T)$.
\end{theorem}

\begin{proof}
Even with perfect instantaneous forgetting prevention ($\Delta_i^t = 0$), each aggregation round introduces drift according to Lemma~\ref{lemma:drift}. Let $\delta^{r,t}$ denote the drift at round $r$ of task $t$:
\begin{equation}
\delta^{r,t} := \mathbb{E}\Big[\|\Theta^{r+1,t} - \hat{\Theta}^{r,t}\|^2\Big] 
\geq \frac{\eta^2\gamma^2}{K},
\end{equation}
where $\hat{\Theta}^{r,t}$ is the ideal model at that round. Accumulating over $T$ tasks and $R$ rounds per task:
\begin{equation}
\sum_{t=1}^T \sum_{r=1}^R \delta^{r,t} \geq T \cdot R \cdot \frac{\eta^2\gamma^2}{K}.
\end{equation}
The memory-buffer coverage effect (Assumption~\ref{asm:memory}) implies diminishing rehearsal effectiveness as $C^o$ grows. Under iCaRL exemplar selection~\cite{Rebuffi2017CVPR}, with only $M/C^o$ samples per previously learned class, the memory-based rehearsal gradient $\nabla_m$ is a finite-sample class-mean approximation whose magnitude scales on the order of $\sqrt{M/C^o}$ relative to the full-data gradient $\nabla_f$:
\begin{equation}
\|\nabla_m\|
\lesssim \sqrt{\frac{M}{C^o}} \|\nabla_f\|.
\end{equation} This coverage limitation amplifies the effective drift. The actual degradation on old tasks must account for both the parameter drift and the diminished rehearsal effectiveness:
\begin{equation}
\begin{aligned}
\Delta^T 
&\geq \Omega\Big(\sum_{t=1}^T \sum_{r=1}^R \delta^{r,t} \cdot \sqrt{\frac{C^o(t)}{M}}\Big) \\
&\geq \Omega\Big(\frac{\eta^2\gamma^2 TR}{K} \cdot \frac{1}{T}\sum_{t=1}^T \sqrt{\frac{C^o(t)}{M}}\Big).
\end{aligned}
\end{equation}
Since $C^o(t) = \sum_{i=0}^{t-1}C^i \geq c \cdot t$ for some constant $c > 0$ (assuming each task introduces at least a constant number of classes), we have:
\begin{equation}
\frac{1}{T}\sum_{t=1}^T \sqrt{C^o(t)} 
\geq \frac{1}{T}\sum_{t=1}^T \sqrt{ct} 
= \Omega(T^{1/2}).
\end{equation}
Substituting back and using $\gamma \geq \gamma_0$ yields:
\begin{equation}
\Delta^T \geq \Omega\Big(\gamma_0^2 \cdot T^{3/2} \cdot R \cdot \frac{1}{\sqrt{M}}\Big).
\end{equation}
When $M$ is a fixed constant (Assumption~\ref{asm:memory}), this simplifies to the stated bound.
\end{proof}

Theorem~\ref{thm:longterm} does not argue against the continual-learning formulation; rather, it indicates that training-time prevention alone is insufficient over an extended mission horizon, so that a complete framework must include a recovery mechanism that operates at the long-term scale. The alternatives, either retraining from scratch at each new region or freezing the model after initial deployment, are ruled out respectively by the uplink budget and by the need to remain safe on newly encountered terrain. This observation is the direct motivation for the Rapid Knowledge Recovery component in Section~\ref{sec:proposed_method}.

\begin{corollary}[Sample Complexity of Recovery]\label{cor:recovery}
Let the classifier degrade to state $\theta_c^*$ after task $t$. To recover to an $\epsilon$-approximation of the optimal classifier $\theta_c^+$ using recovery function $\psi$, the required memory buffer size and fine-tuning complexity satisfy:
\begin{equation}
\begin{aligned}
|\mathcal{M}_k| &= \Omega\Big(\frac{C^o \cdot d_c}{\epsilon^2}\Big), \\
E_f &= \mathcal{O}\Big(\log\Big(\frac{1}{\epsilon}\Big)\Big),
\end{aligned}
\end{equation}
where $E_f$ denotes the number of fine-tuning epochs. In contrast, complete retraining from scratch requires $E_r = \Omega(C^o)$ epochs to achieve the same accuracy level.
\end{corollary}

\begin{proof}
The recovery function $\psi$ learns to generate corrections through episodic meta-learning:
\begin{equation}
\theta_c^+ = \theta_c^* + \Delta\theta_c, 
\quad \Delta\theta_c = \psi(\theta_c^*, \mathbf{p}, \mathcal{M}_k, \mathbf{a}; \xi).
\end{equation}
By Rademacher-complexity-based uniform-convergence results for multi-output function learning~\cite[Theorem~6.8 and Chapter~26]{ShalevShwartz2014}, to learn a mapping from the combined input space (degraded parameters, prototypes, memory samples, geometric features) to a $d_c$-dimensional output space with $C^o$ output classes at error $\epsilon$, the required sample complexity is:
\begin{equation}
|\mathcal{M}_k| = \Omega\Big(\frac{C^o \cdot d_c}{\epsilon^2}\Big),
\end{equation}
where the complexity scales with $C^o \cdot d_c$ as the recovery function
must distinguish correction patterns for $C^o$ classes across $d_c$ dimensions. This is an informal order-of-magnitude sketch; the sharpest hypothesis-class-dependent constants are not pursued here.

After applying the learned correction $\Delta\theta_c$, fine-tuning on $\mathcal{M}_k$ exhibits exponential convergence due to the warm-start initialization. By strong convexity of the fine-tuning objective in a neighborhood of $\theta_c^+$:
\begin{equation}
\|\theta_c^E - \theta_c^+\|^2 \leq (1-\mu)^E \|\Delta\theta_c\|^2,
\end{equation}
where $\mu > 0$ is the strong convexity parameter and $E$ is the number of fine-tuning epochs. Setting $(1-\mu)^E \leq \epsilon$ gives $E = \mathcal{O}(\log(1/\epsilon))$.

For complete retraining, the model must relearn decision boundaries for all $C^o$ classes from scratch. Even with memory buffer $\mathcal{M}_k$, standard training requires traversing the data multiple times to converge, yielding $E_r = \Omega(C^o)$ complexity in the worst case.
\end{proof}

\begin{remark}\label{remark:recovery}
The logarithmic fine-tuning bound in Corollary~\ref{cor:recovery} uses strong convexity of the fine-tuning objective in a neighborhood of $\theta_c^{+}$. This condition is an approximation in the deep-network setting: the objective is only locally well-behaved around a warm-started initialization. Corollary~\ref{cor:recovery} should therefore be read as a directional efficiency statement, and the empirical recovery curves in Section~\ref{sec:experiments} provide the corresponding validation.
\end{remark}
\section{Performance Evaluation}
\label{sec:experiments}

\subsection{Datasets and Evaluation Metrics}

To evaluate our FCL framework under realistic planetary exploration conditions, we conduct experiments on three publicly available Mars terrain segmentation datasets that represent different operational contexts and domain characteristics.

\textbf{MarsScapes} \cite{Liu2023AA} is the first panoramic dataset designed for Martian terrain understanding. It consists of 18,460 images at a resolution of 512 by 512, all collected by the Curiosity rover at Gale Crater. The dataset includes pixel-level annotations for nine semantic categories: Soil, Sand, Gravel, Bedrock, Rocks, Tracks, Shadows, Background, and Unknown. These categories are essential for assessing rover traversability, since loose materials can increase mobility risks, while consolidated surfaces are safer for navigation. S5Mars \cite{Zhang2024TGRS} contains 6,000 high-resolution images at 1200 by 1200, also from Curiosity's Mastcam. It covers nine categories, including soil, sand, bedrock, rock, ridge, trace, hole, sky, and rover. Notably, \textbf{S5Mars} introduces navigation-critical classes such as hole, which represents negative obstacles, and ridge, which indicates terrain undulation. Together, MarsScapes and S5Mars support a thorough evaluation of continual learning within a single mission. To evaluate generalization across missions, we use \textbf{AI4MARS} \cite{Swan2021CVPR}, which is NASA's largest Mars terrain dataset. AI4MARS contains about 35,000 images from the Spirit, Opportunity, and Curiosity rovers, spanning three different landing sites. It provides four categories that are critical for traversability: Soil, Bedrock, Sand, and Big Rock. The domain shift between geological regions and rover platforms in AI4MARS introduces significant data heterogeneity, which presents challenges for federated learning.


We consider two experimental settings on MarsScapes and S5Mars. The first, referred to as the 5-1 setting, involves learning 5 base classes, followed by 4 new tasks, each introducing 1 additional class ($T=5$). The second, the 3-2 setting, starts with 3 base classes and then adds 3 new tasks, each with 2 classes ($T=4$). For AI4MARS, we use a 2-1 setting, where 2 base classes are learned first, followed by 2 new tasks with 1 class each ($T=3$). This setup is intended to simulate cross-mission knowledge integration, where new rover deployments encounter terrain types not present in earlier missions. As the main evaluation metric, we use mean Intersection over Union (mIoU) calculated over all classes after the final task. We also report per-class IoU to examine how forgetting patterns vary in different terrain categories.


\begin{table*}[ht]
\centering
\setlength{\tabcolsep}{2.2mm}
\caption{Comparisons of mIoU (\%) on MarsScapes dataset \cite{Liu2023AA} under the setting of 5-1. Class IDs: 0-Soil, 1-Sand, 2-Bedrock, 3-Gravel, 4-Rocks, 5-Tracks, 6-Shadows, 7-Background, 8-Unknown.}
\resizebox{\linewidth}{!}{
	\begin{tabular}{c|ccccccccc|>{\columncolor{lightgray}}c|>{\columncolor{lightgray}}c}
		\toprule
		Class ID & 0 & 1 & 2 & 3 & 4 & 5 & 6 & 7 & 8 & mIoU & Imp. \\
		\midrule
		Centralized (upper bound) & 57.2$_{\pm.14}$ & 53.8$_{\pm.19}$ & 41.6$_{\pm.17}$ & 41.5$_{\pm.21}$ & 28.4$_{\pm.24}$ & 25.6$_{\pm.13}$ & 14.8$_{\pm.22}$ & 15.2$_{\pm.16}$ & 15.3$_{\pm.25}$ & 32.6 & ref. \\
		\midrule
		Finetuning + FL & 32.6$_{\pm.26}$ & 18.4$_{\pm.19}$ & 0.0$_{\pm.00}$ & 0.0$_{\pm.00}$ & 0.0$_{\pm.00}$ & 0.0$_{\pm.00}$ & 0.0$_{\pm.00}$ & 0.0$_{\pm.00}$ & 4.2$_{\pm.23}$ & 6.1 & $\Uparrow$ 24.3 \\
		PLOP \cite{Douillard2021CVPR} + FL & 43.2$_{\pm.18}$ & 33.4$_{\pm.24}$ & 26.6$_{\pm.13}$ & 19.5$_{\pm.27}$ & 14.2$_{\pm.16}$ & 5.3$_{\pm.31}$ & 1.8$_{\pm.19}$ & 0.0$_{\pm.28}$ & 4.1$_{\pm.22}$ & 16.5 & $\Uparrow$ 13.9 \\
		CUE \cite{Lin2025CVPR} + FL & 46.3$_{\pm.21}$ & 36.2$_{\pm.15}$ & 29.8$_{\pm.28}$ & 23.6$_{\pm.17}$ & 18.1$_{\pm.24}$ & 8.5$_{\pm.12}$ & 4.2$_{\pm.29}$ & 1.4$_{\pm.16}$ & 6.8$_{\pm.25}$ & 19.4 & $\Uparrow$ 11.0 \\
		CS$^2$K \cite{Cong2024ECCV} + FL & 48.5$_{\pm.14}$ & 38.8$_{\pm.27}$ & 32.2$_{\pm.19}$ & 26.4$_{\pm.23}$ & 20.6$_{\pm.11}$ & 11.1$_{\pm.26}$ & 6.3$_{\pm.17}$ & 3.1$_{\pm.31}$ & 8.5$_{\pm.14}$ & 21.7 & $\Uparrow$ 8.7 \\
		CoMBO \cite{Fang2025CVPR} + FL & 50.1$_{\pm.23}$ & 40.6$_{\pm.16}$ & 33.9$_{\pm.21}$ & 28.2$_{\pm.14}$ & 22.4$_{\pm.28}$ & 13.2$_{\pm.17}$ & \textcolor{deepred}{\textbf{12.5}}$_{\pm.25}$ & 4.3$_{\pm.19}$ & 5.4$_{\pm.32}$ & 23.4 & $\Uparrow$ 7.0 \\
		ADAPT \cite{Yang2025ICLR} + FL & 51.8$_{\pm.17}$ & 42.2$_{\pm.29}$ & 35.5$_{\pm.14}$ & 30.1$_{\pm.26}$ & \textcolor{deepred}{\textbf{26.2}}$_{\pm.18}$ & 14.6$_{\pm.23}$ & 8.8$_{\pm.12}$ & 5.8$_{\pm.27}$ & 10.2$_{\pm.21}$ & 25.0 & $\Uparrow$ 5.4 \\
		EIR \cite{Yin2025CVPR} + FL & \textcolor{deepred}{\textbf{55.2}}$_{\pm.12}$ & 43.5$_{\pm.21}$ & 37.3$_{\pm.26}$ & 31.4$_{\pm.15}$ & 24.3$_{\pm.23}$ & 16.4$_{\pm.19}$ & 9.6$_{\pm.28}$ & 7.1$_{\pm.14}$ & 11.8$_{\pm.17}$ & 26.3 & $\Uparrow$ 4.1 \\
		FBL \cite{Dong2023CVPR} & 53.6$_{\pm.24}$ & \textcolor{blue}{\textbf{45.2}}$_{\pm.13}$ & \textcolor{deepred}{\textbf{39.5}}$_{\pm.18}$ & \textcolor{blue}{\textbf{33.2}}$_{\pm.29}$ & 25.1$_{\pm.15}$ & \textcolor{blue}{\textbf{18.5}}$_{\pm.22}$ & \textcolor{blue}{\textbf{11.1}}$_{\pm.16}$ & \textcolor{blue}{\textbf{8.4}}$_{\pm.25}$ & \textcolor{deepred}{\textbf{13.6}}$_{\pm.11}$ & \textcolor{blue}{\textbf{27.6}} & $\Uparrow$ 2.8 \\
		\midrule
		\textbf{Ours} & \textcolor{blue}{\textbf{55.0}}$_{\pm.16}$ & \textcolor{deepred}{\textbf{51.8}}$_{\pm.22}$ & \textcolor{blue}{\textbf{39.2}}$_{\pm.11}$ & \textcolor{deepred}{\textbf{39.0}}$_{\pm.19}$ & \textcolor{blue}{\textbf{26.0}}$_{\pm.27}$ & \textcolor{deepred}{\textbf{23.7}}$_{\pm.14}$ & 12.4$_{\pm.23}$ & \textcolor{deepred}{\textbf{13.0}}$_{\pm.18}$ & \textcolor{blue}{\textbf{13.4}}$_{\pm.26}$ & \textcolor{deepred}{\textbf{30.4}} & -- \\
		\bottomrule
	\end{tabular}}
\label{tab:comparison_marsscapes_5_1}
\end{table*}

\begin{table*}[t]
\centering
\setlength{\tabcolsep}{2.2mm}
\caption{Comparisons of mIoU (\%) on MarsScapes dataset \cite{Liu2023AA} under the setting of 3-2.}
\resizebox{\linewidth}{!}{
	\begin{tabular}{c|ccccccccc|>{\columncolor{lightgray}}c|>{\columncolor{lightgray}}c}
		\toprule
		Class ID & 0 & 1 & 2 & 3 & 4 & 5 & 6 & 7 & 8 & mIoU & Imp. \\
		\midrule
		Centralized (upper bound) & 58.5$_{\pm.15}$ & 45.8$_{\pm.22}$ & 42.0$_{\pm.13}$ & 32.5$_{\pm.18}$ & 29.0$_{\pm.25}$ & 19.4$_{\pm.14}$ & 15.5$_{\pm.20}$ & 11.5$_{\pm.23}$ & 16.8$_{\pm.16}$ & 30.1 & ref. \\
		\midrule
		Finetuning + FL & 46.8$_{\pm.27}$ & 16.2$_{\pm.23}$ & 0.0$_{\pm.00}$ & 0.0$_{\pm.00}$ & 0.0$_{\pm.00}$ & 8.5$_{\pm.21}$ & 0.0$_{\pm.00}$ & 3.8$_{\pm.19}$ & 5.2$_{\pm.24}$ & 8.9 & $\Uparrow$ 18.9 \\
		PLOP \cite{Douillard2021CVPR} + FL & 47.8$_{\pm.23}$ & 30.4$_{\pm.16}$ & 23.2$_{\pm.28}$ & 16.3$_{\pm.14}$ & 10.6$_{\pm.25}$ & 3.4$_{\pm.19}$ & 0.0$_{\pm.00}$ & 0.0$_{\pm.00}$ & 2.8$_{\pm.27}$ & 14.9 & $\Uparrow$ 12.9 \\
		CUE \cite{Lin2025CVPR} + FL & 49.6$_{\pm.19}$ & 33.2$_{\pm.26}$ & 26.5$_{\pm.13}$ & 20.1$_{\pm.24}$ & 14.4$_{\pm.17}$ & 6.2$_{\pm.29}$ & 2.6$_{\pm.15}$ & 0.0$_{\pm.00}$ & 5.1$_{\pm.22}$ & 17.5 & $\Uparrow$ 10.3 \\
		CS$^2$K \cite{Cong2024ECCV} + FL & 51.2$_{\pm.27}$ & 35.5$_{\pm.14}$ & 29.1$_{\pm.21}$ & 22.8$_{\pm.18}$ & 17.1$_{\pm.31}$ & 8.8$_{\pm.16}$ & 4.5$_{\pm.24}$ & 1.8$_{\pm.12}$ & 0.0$_{\pm.00}$ & 19.0 & $\Uparrow$ 8.8 \\
		CoMBO \cite{Fang2025CVPR} + FL & 46.8$_{\pm.16}$ & 37.6$_{\pm.28}$ & 30.8$_{\pm.17}$ & 24.5$_{\pm.23}$ & 19.2$_{\pm.14}$ & 0.0$_{\pm.00}$ & 7.1$_{\pm.21}$ & \textcolor{deepred}{\textbf{9.2}}$_{\pm.32}$ & 3.7$_{\pm.18}$ & 19.9 & $\Uparrow$ 7.9 \\
		ADAPT \cite{Yang2025ICLR} + FL & 50.4$_{\pm.24}$ & \textcolor{deepred}{\textbf{44.3}}$_{\pm.17}$ & 32.2$_{\pm.29}$ & \textcolor{deepred}{\textbf{30.6}}$_{\pm.13}$ & 20.4$_{\pm.22}$ & 12.1$_{\pm.15}$ & 6.2$_{\pm.27}$ & 4.4$_{\pm.19}$ & 8.5$_{\pm.31}$ & 23.2 & $\Uparrow$ 4.6 \\
		EIR \cite{Yin2025CVPR} + FL & 51.8$_{\pm.15}$ & 40.8$_{\pm.23}$ & 33.6$_{\pm.18}$ & 27.4$_{\pm.27}$ & \textcolor{blue}{\textbf{23.1}}$_{\pm.12}$ & \textcolor{deepred}{\textbf{17.7}}$_{\pm.24}$ & \textcolor{blue}{\textbf{9.3}}$_{\pm.19}$ & 5.6$_{\pm.28}$ & 9.8$_{\pm.16}$ & 24.3 & $\Uparrow$ 3.5 \\
		FBL \cite{Dong2023CVPR} & \textcolor{blue}{\textbf{52.2}}$_{\pm.21}$ & 41.5$_{\pm.12}$ & \textcolor{blue}{\textbf{35.2}}$_{\pm.25}$ & 28.2$_{\pm.16}$ & 21.8$_{\pm.29}$ & 14.8$_{\pm.18}$ & 7.6$_{\pm.13}$ & \textcolor{blue}{\textbf{7.4}}$_{\pm.24}$ & \textcolor{blue}{\textbf{11.1}}$_{\pm.21}$ & \textcolor{blue}{\textbf{24.4}} & $\Uparrow$ 3.4 \\
		\midrule
		\textbf{Ours} & \textcolor{deepred}{\textbf{56.2}}$_{\pm.18}$ & \textcolor{blue}{\textbf{43.6}}$_{\pm.27}$ & \textcolor{deepred}{\textbf{39.8}}$_{\pm.14}$ & \textcolor{blue}{\textbf{30.2}}$_{\pm.22}$ & \textcolor{deepred}{\textbf{26.8}}$_{\pm.17}$ & \textcolor{blue}{\textbf{17.0}}$_{\pm.31}$ & \textcolor{deepred}{\textbf{13.2}}$_{\pm.23}$ & 9.0$_{\pm.15}$ & \textcolor{deepred}{\textbf{14.4}}$_{\pm.19}$ & \textcolor{deepred}{\textbf{27.8}} & -- \\
		\bottomrule
	\end{tabular}}
\label{tab:comparison_marsscapes_3_2}
\end{table*}

\begin{table*}[t]
\centering
\setlength{\tabcolsep}{2.2mm}
\caption{Comparisons of mIoU (\%) on S5Mars dataset \cite{Zhang2024TGRS} under the setting of 5-1. Class IDs: 0-soil, 1-sand, 2-bedrock, 3-rock, 4-ridge, 5-trace, 6-hole, 7-sky, 8-rover.}
\resizebox{\linewidth}{!}{
	\begin{tabular}{c|ccccccccc|>{\columncolor{lightgray}}c|>{\columncolor{lightgray}}c}
		\toprule
		Class ID & 0 & 1 & 2 & 3 & 4 & 5 & 6 & 7 & 8 & mIoU & Imp. \\
		\midrule
		Centralized (upper bound) & 74.8$_{\pm.11}$ & 66.8$_{\pm.22}$ & 58.2$_{\pm.16}$ & 43.5$_{\pm.25}$ & 36.8$_{\pm.19}$ & 39.6$_{\pm.13}$ & 28.0$_{\pm.26}$ & 25.6$_{\pm.15}$ & 27.2$_{\pm.21}$ & 44.5 & ref. \\
		\midrule
		Finetuning + FL & 51.2$_{\pm.26}$ & 22.4$_{\pm.19}$ & 8.6$_{\pm.23}$ & 0.0$_{\pm.00}$ & 0.0$_{\pm.00}$ & 5.8$_{\pm.21}$ & 0.0$_{\pm.00}$ & 0.0$_{\pm.00}$ & 6.2$_{\pm.25}$ & 10.5 & $\Uparrow$ 31.7 \\
		PLOP \cite{Douillard2021CVPR} + FL & 62.4$_{\pm.19}$ & 48.6$_{\pm.26}$ & 42.4$_{\pm.14}$ & 35.8$_{\pm.31}$ & \textcolor{deepred}{\textbf{36.4}}$_{\pm.22}$ & 20.2$_{\pm.18}$ & 0.0$_{\pm.00}$ & 12.8$_{\pm.15}$ & 0.0$_{\pm.00}$ & 28.7 & $\Uparrow$ 13.5 \\
		CUE \cite{Lin2025CVPR} + FL & 65.8$_{\pm.23}$ & 52.2$_{\pm.14}$ & 45.8$_{\pm.28}$ & \textcolor{deepred}{\textbf{43.4}}$_{\pm.17}$ & 33.2$_{\pm.25}$ & 0.0$_{\pm.00}$ & 0.0$_{\pm.00}$ & 18.4$_{\pm.32}$ & 17.6$_{\pm.19}$ & 30.7 & $\Uparrow$ 11.5 \\
		CS$^2$K \cite{Cong2024ECCV} + FL & 67.4$_{\pm.16}$ & 55.6$_{\pm.29}$ & 48.6$_{\pm.21}$ & 39.4$_{\pm.24}$ & 34.8$_{\pm.13}$ & 10.8$_{\pm.18}$ & \textcolor{deepred}{\textbf{27.8}}$_{\pm.17}$ & 22.6$_{\pm.26}$ & 19.4$_{\pm.14}$ & 36.3 & $\Uparrow$ 5.9 \\
		CoMBO \cite{Fang2025CVPR} + FL & 64.2$_{\pm.27}$ & 57.4$_{\pm.18}$ & 52.6$_{\pm.23}$ & 40.2$_{\pm.15}$ & 32.6$_{\pm.31}$ & 28.4$_{\pm.14}$ & 23.8$_{\pm.26}$ & \textcolor{deepred}{\textbf{26.8}}$_{\pm.19}$ & 0.0$_{\pm.00}$ & 36.2 & $\Uparrow$ 6.0 \\
		ADAPT \cite{Yang2025ICLR} + FL & 68.6$_{\pm.14}$ & 61.6$_{\pm.23}$ & 51.6$_{\pm.17}$ & 41.6$_{\pm.28}$ & 33.8$_{\pm.19}$ & \textcolor{blue}{\textbf{35.6}}$_{\pm.26}$ & 25.2$_{\pm.22}$ & 20.4$_{\pm.13}$ & \textcolor{deepred}{\textbf{25.8}}$_{\pm.31}$ & 40.5 & $\Uparrow$ 1.7 \\
		EIR \cite{Yin2025CVPR} + FL & \textcolor{blue}{\textbf{70.8}}$_{\pm.21}$ & 59.2$_{\pm.15}$ & \textcolor{blue}{\textbf{54.2}}$_{\pm.29}$ & \textcolor{blue}{\textbf{42.8}}$_{\pm.18}$ & \textcolor{blue}{\textbf{35.6}}$_{\pm.24}$ & 32.4$_{\pm.17}$ & 25.6$_{\pm.31}$ & \textcolor{blue}{\textbf{24.2}}$_{\pm.22}$ & 22.4$_{\pm.16}$ & \textcolor{blue}{\textbf{40.8}} & $\Uparrow$ 1.4 \\
		FBL \cite{Dong2023CVPR} & 69.2$_{\pm.18}$ & \textcolor{blue}{\textbf{62.4}}$_{\pm.27}$ & 53.8$_{\pm.14}$ & 39.6$_{\pm.23}$ & 34.2$_{\pm.16}$ & 33.8$_{\pm.32}$ & \textcolor{blue}{\textbf{26.5}}$_{\pm.19}$ & 21.8$_{\pm.25}$ & 24.2$_{\pm.21}$ & 40.6 & $\Uparrow$ 1.6 \\
		\midrule
		\textbf{Ours} & \textcolor{deepred}{\textbf{72.5}}$_{\pm.13}$ & \textcolor{deepred}{\textbf{64.5}}$_{\pm.24}$ & \textcolor{deepred}{\textbf{55.8}}$_{\pm.18}$ & 41.2$_{\pm.27}$ & 34.6$_{\pm.21}$ & \textcolor{deepred}{\textbf{37.4}}$_{\pm.15}$ & 25.8$_{\pm.28}$ & 23.6$_{\pm.17}$ & \textcolor{blue}{\textbf{24.8}}$_{\pm.23}$ & \textcolor{deepred}{\textbf{42.2}} & -- \\
		\bottomrule
	\end{tabular}}
\label{tab:comparison_s5mars_5_1}
\end{table*}

\begin{table*}[t]
\centering
\setlength{\tabcolsep}{2.2mm}
\caption{Comparisons of mIoU (\%) on S5Mars dataset \cite{Zhang2024TGRS} under the setting of 3-2.}
\resizebox{\linewidth}{!}{
	\begin{tabular}{c|ccccccccc|>{\columncolor{lightgray}}c|>{\columncolor{lightgray}}c}
		\toprule
		Class ID & 0 & 1 & 2 & 3 & 4 & 5 & 6 & 7 & 8 & mIoU & Imp. \\
		\midrule
		Centralized (upper bound) & 72.5$_{\pm.14}$ & 62.5$_{\pm.21}$ & 54.2$_{\pm.17}$ & 33.8$_{\pm.19}$ & 29.0$_{\pm.24}$ & 41.8$_{\pm.12}$ & 34.2$_{\pm.26}$ & 17.8$_{\pm.16}$ & 27.7$_{\pm.22}$ & 41.5 & ref. \\
		\midrule
		Finetuning + FL & 48.4$_{\pm.28}$ & 18.6$_{\pm.24}$ & 0.0$_{\pm.00}$ & 0.0$_{\pm.00}$ & 0.0$_{\pm.00}$ & 7.2$_{\pm.22}$ & 0.0$_{\pm.00}$ & 4.5$_{\pm.18}$ & 5.8$_{\pm.26}$ & 9.4 & $\Uparrow$ 29.8 \\
		PLOP \cite{Douillard2021CVPR} + FL & 51.8$_{\pm.22}$ & 37.4$_{\pm.17}$ & 31.2$_{\pm.28}$ & 24.6$_{\pm.14}$ & 20.4$_{\pm.26}$ & 10.6$_{\pm.31}$ & 0.0$_{\pm.00}$ & 6.8$_{\pm.24}$ & 6.4$_{\pm.21}$ & 21.0 & $\Uparrow$ 18.2 \\
		CUE \cite{Lin2025CVPR} + FL & 55.2$_{\pm.16}$ & 41.6$_{\pm.28}$ & 34.8$_{\pm.19}$ & \textcolor{deepred}{\textbf{32.8}}$_{\pm.25}$ & 22.6$_{\pm.13}$ & 0.0$_{\pm.00}$ & 0.0$_{\pm.00}$ & 10.2$_{\pm.17}$ & 8.2$_{\pm.31}$ & 22.8 & $\Uparrow$ 16.4 \\
		CS$^2$K \cite{Cong2024ECCV} + FL & 57.8$_{\pm.24}$ & 44.8$_{\pm.15}$ & 38.2$_{\pm.27}$ & 29.4$_{\pm.18}$ & \textcolor{deepred}{\textbf{27.8}}$_{\pm.21}$ & 15.6$_{\pm.16}$ & 13.6$_{\pm.32}$ & \textcolor{deepred}{\textbf{16.8}}$_{\pm.23}$ & 0.0$_{\pm.00}$ & 27.1 & $\Uparrow$ 12.1 \\
		CoMBO \cite{Fang2025CVPR} + FL & 53.6$_{\pm.18}$ & 46.6$_{\pm.26}$ & \textcolor{blue}{\textbf{44.4}}$_{\pm.14}$ & 28.8$_{\pm.31}$ & 21.8$_{\pm.17}$ & 0.0$_{\pm.00}$ & \textcolor{blue}{\textbf{18.8}}$_{\pm.21}$ & 11.6$_{\pm.15}$ & 8.8$_{\pm.27}$ & 26.0 & $\Uparrow$ 13.2 \\
		ADAPT \cite{Yang2025ICLR} + FL & 58.8$_{\pm.27}$ & \textcolor{blue}{\textbf{52.2}}$_{\pm.13}$ & 40.6$_{\pm.22}$ & 31.4$_{\pm.16}$ & 24.4$_{\pm.29}$ & \textcolor{blue}{\textbf{25.4}}$_{\pm.18}$ & 15.2$_{\pm.24}$ & 13.4$_{\pm.31}$ & 13.4$_{\pm.19}$ & 30.5 & $\Uparrow$ 8.7 \\
		EIR \cite{Yin2025CVPR} + FL & \textcolor{blue}{\textbf{60.6}}$_{\pm.15}$ & 48.4$_{\pm.24}$ & 42.2$_{\pm.18}$ & \textcolor{blue}{\textbf{32.2}}$_{\pm.27}$ & 23.8$_{\pm.14}$ & 21.6$_{\pm.23}$ & 17.4$_{\pm.31}$ & 15.6$_{\pm.16}$ & \textcolor{blue}{\textbf{14.8}}$_{\pm.28}$ & 30.7 & $\Uparrow$ 8.5 \\
		FBL \cite{Dong2023CVPR} & 59.2$_{\pm.21}$ & 50.6$_{\pm.19}$ & 43.4$_{\pm.26}$ & 30.2$_{\pm.14}$ & \textcolor{blue}{\textbf{26.2}}$_{\pm.23}$ & 23.8$_{\pm.17}$ & 16.2$_{\pm.28}$ & 14.8$_{\pm.22}$ & 13.6$_{\pm.15}$ & \textcolor{blue}{\textbf{30.9}} & $\Uparrow$ 8.3 \\
		\midrule
		\textbf{Ours} & \textcolor{deepred}{\textbf{70.2}}$_{\pm.19}$ & \textcolor{deepred}{\textbf{60.2}}$_{\pm.28}$ & \textcolor{deepred}{\textbf{51.8}}$_{\pm.15}$ & 31.6$_{\pm.22}$ & 26.8$_{\pm.27}$ & \textcolor{deepred}{\textbf{39.6}}$_{\pm.14}$ & \textcolor{deepred}{\textbf{31.8}}$_{\pm.23}$ & \textcolor{blue}{\textbf{15.8}}$_{\pm.18}$ & \textcolor{deepred}{\textbf{25.0}}$_{\pm.26}$ & \textcolor{deepred}{\textbf{39.2}} & -- \\
		\bottomrule
	\end{tabular}}
\label{tab:comparison_s5mars_3_2}
\end{table*}

\subsection{Implementation Details}

For fair comparison with state-of-the-art continual learning methods under federated settings, we adapt six representative segmentation
approaches \cite{Douillard2021CVPR, Cong2024ECCV, Fang2025CVPR, Yang2025ICLR, Yin2025CVPR, Lin2025CVPR}
to federated learning by incorporating local training and federated aggregation protocols,
along with one federated segmentation method \cite{Dong2023CVPR} and a
standard fine-tuning baseline. All adapted baselines share the same FedAvg aggregation as our framework and keep their method-specific mechanisms strictly client-local: PLOP's~\cite{Douillard2021CVPR} KD teacher is the client's own previous-task model; the exemplar memory in CoMBO~\cite{Fang2025CVPR}, CS$^2$K~\cite{Cong2024ECCV}, EIR~\cite{Yin2025CVPR}, CUE~\cite{Lin2025CVPR}, and ADAPT~\cite{Yang2025ICLR} is per-client and never exchanged; FBL~\cite{Dong2023CVPR} is used in its native federated form. All baselines therefore operate under identical communication budgets. Additionally, to evaluate long-term recovery efficiency, we compare against three recovery strategies: low-rank adaptation \cite{He2025CVPR}, 
meta-learning based episodic distillation \cite{Wang2022CVPR}, and variational 
knowledge distillation \cite{Li2023TPAMI}. All methods employ the identical 
segmentation backbone DeepLabV3+ \cite{Chen2018ECCV} with ResNet-50 \cite{He2016CVPR} 
pretrained on ImageNet \cite{Russakovsky2015IJCV} to ensure fair evaluation. We 
follow standard data augmentation strategies, including random horizontal flip, 
random scaling, and random cropping for all experiments. Each client maintains 
an exemplar memory buffer $\mathcal{M}_k$ with a fixed capacity of 200 samples, 
following the iCaRL protocol \cite{Rebuffi2017CVPR} to store exemplars whose 
features are closest to the class mean representation. As new classes arrive, 
we allocate $\frac{|\mathcal{M}_k|}{C^o+C^t}$ exemplars per class and remove 
$\frac{|\mathcal{M}_k|}{C^o}-\frac{|\mathcal{M}_k|}{C^o+C^t}$ samples per old 
class to maintain the fixed buffer size.


\begin{table*}[t]
\centering
\setlength{\tabcolsep}{4mm}
\caption{Comparisons of mIoU (\%) on AI4MARS dataset \cite{Swan2021CVPR} under the setting of 2-1. Class IDs: 0-Soil, 1-Bedrock, 2-Sand, 3-Big Rock.}
\resizebox{0.85\linewidth}{!}{
	\begin{tabular}{c|cccc|>{\columncolor{lightgray}}c|>{\columncolor{lightgray}}c}
		\toprule
		Class ID & 0 & 1 & 2 & 3 & mIoU & Imp. \\
		\midrule
		Centralized (upper bound) & 60.2$_{\pm.16}$ & 51.0$_{\pm.22}$ & 15.0$_{\pm.19}$ & 18.6$_{\pm.14}$ & 36.2 & ref. \\
		\midrule
		Finetuning + FL & 35.4$_{\pm.28}$ & 1.8$_{\pm.21}$ & 0.0$_{\pm.00}$ & 0.0$_{\pm.00}$ & 9.3 & $\Uparrow$ 24.2 \\
		PLOP \cite{Douillard2021CVPR} + FL & 42.8$_{\pm.19}$ & 13.4$_{\pm.27}$ & 0.0$_{\pm.00}$ & 0.0$_{\pm.00}$ & 14.1 & $\Uparrow$ 19.4 \\
		CUE \cite{Lin2025CVPR} + FL & 46.2$_{\pm.24}$ & 19.8$_{\pm.16}$ & 0.0$_{\pm.00}$ & 0.0$_{\pm.00}$ & 16.5 & $\Uparrow$ 17.0 \\
		CS$^2$K \cite{Cong2024ECCV} + FL & 50.6$_{\pm.17}$ & 26.2$_{\pm.31}$ & 3.4$_{\pm.22}$ & 0.0$_{\pm.00}$ & 20.1 & $\Uparrow$ 13.4 \\
		CoMBO \cite{Fang2025CVPR} + FL & 47.8$_{\pm.26}$ & 38.4$_{\pm.14}$ & 4.2$_{\pm.29}$ & 0.0$_{\pm.00}$ & 22.6 & $\Uparrow$ 10.9 \\
		ADAPT \cite{Yang2025ICLR} + FL & 54.2$_{\pm.21}$ & 33.6$_{\pm.18}$ & 8.2$_{\pm.25}$ & 0.0$_{\pm.00}$ & 24.0 & $\Uparrow$ 9.5 \\
		EIR \cite{Yin2025CVPR} + FL & \textcolor{deepred}{\textbf{59.6}}$_{\pm.15}$ & 32.4$_{\pm.28}$ & 7.6$_{\pm.17}$ & 0.0$_{\pm.00}$ & 24.9 & $\Uparrow$ 8.6 \\
		FBL \cite{Dong2023CVPR} & 56.8$_{\pm.23}$ & \textcolor{blue}{\textbf{42.8}}$_{\pm.19}$ & \textcolor{blue}{\textbf{10.4}}$_{\pm.31}$ & \textcolor{blue}{\textbf{8.2}}$_{\pm.26}$ & \textcolor{blue}{\textbf{29.6}} & $\Uparrow$ 3.9 \\
		\midrule
		\textbf{Ours} & \textcolor{blue}{\textbf{58.4}}$_{\pm.18}$ & \textcolor{deepred}{\textbf{48.2}}$_{\pm.24}$ & \textcolor{deepred}{\textbf{12.6}}$_{\pm.21}$ & \textcolor{deepred}{\textbf{14.6}}$_{\pm.16}$ & \textcolor{deepred}{\textbf{33.5}} & -- \\
		\bottomrule
	\end{tabular}}
\label{tab:comparison_ai4mars_2_1}
\end{table*}

\begin{figure*}[t]
	\centering
 	\includegraphics[width=0.95\textwidth]{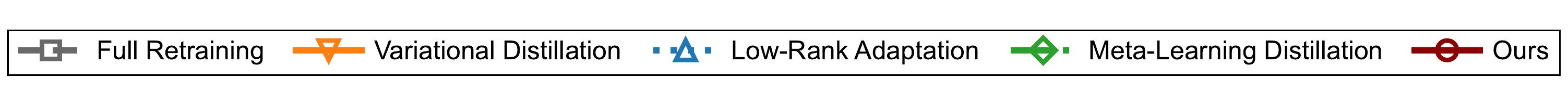}
	
	\vspace{-5pt}
	
	\begin{subfigure}[b]{0.32\textwidth}
		\centering
		\includegraphics[width=\textwidth]{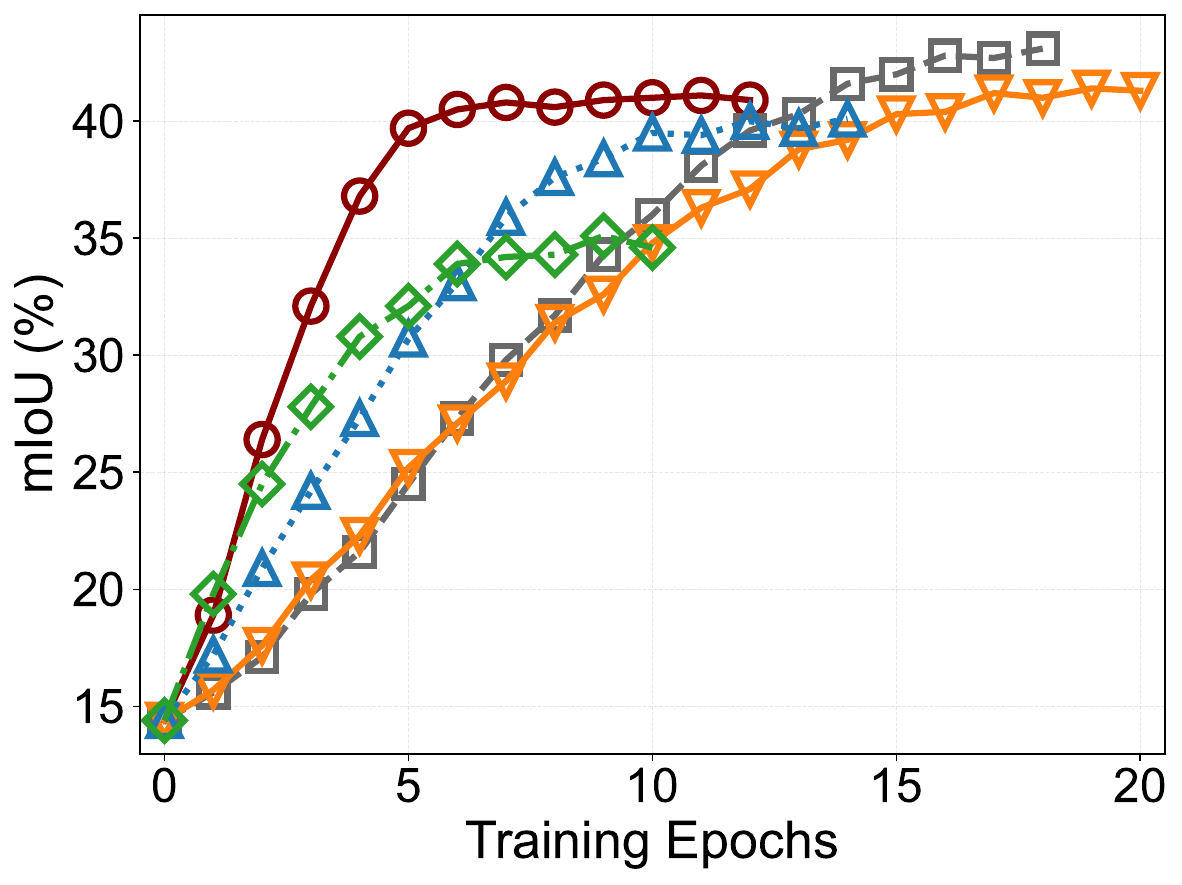}
		\caption{MarsScapes}
		\label{fig:recovery_marsscapes}
	\end{subfigure}
	\hfill
	\begin{subfigure}[b]{0.32\textwidth}
		\centering
		\includegraphics[width=\textwidth]{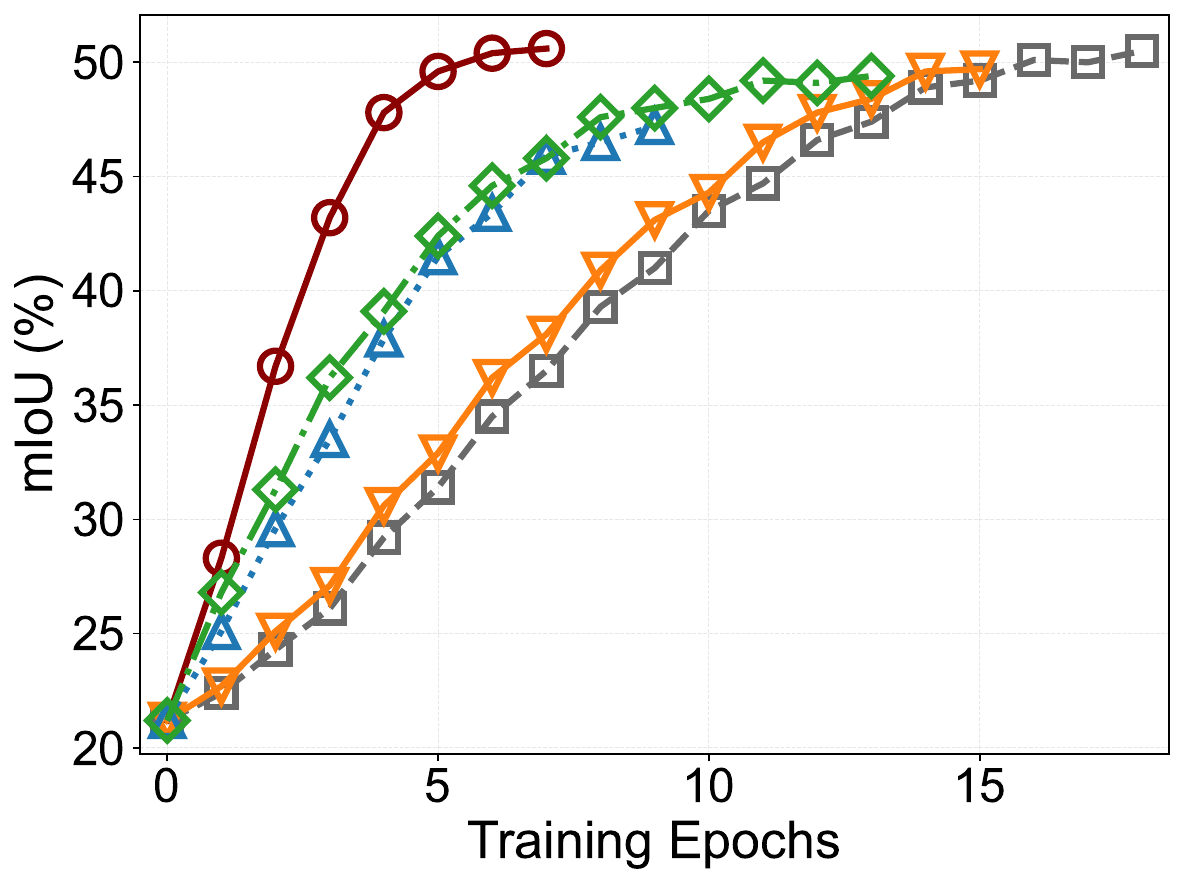}
		\caption{S5Mars}
		\label{fig:recovery_s5mars}
	\end{subfigure}
	\hfill
	\begin{subfigure}[b]{0.32\textwidth}
		\centering
		\includegraphics[width=\textwidth]{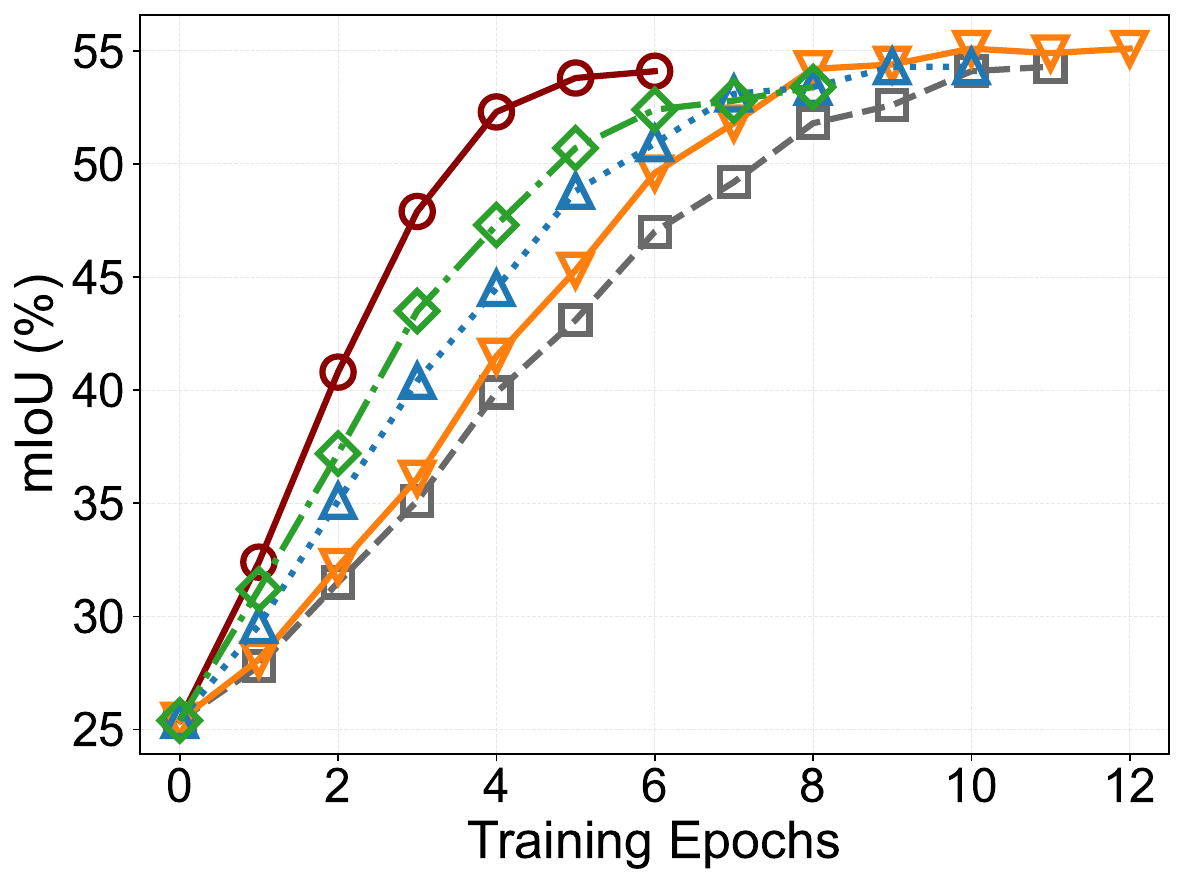}
		\caption{AI4MARS}
		\label{fig:recovery_ai4mars}
	\end{subfigure}
	
	\vspace{-2pt}
	\caption{Recovery efficiency comparison across three Mars terrain datasets.}
	\label{fig:recovery_efficiency}
\end{figure*}

The Stochastic Gradient Descent (SGD) optimizer is used here, with a momentum value of 0.9 and a weight decay of $1\times10^{-4}$, in order to encourage stable convergence and regularization during training. For the base task, the learning rate is initialized at $1\times10^{-2}$, and is subsequently reduced for incremental tasks to better accommodate the changing data distribution. Specifically, a learning rate of $1\times10^{-3}$ is used for all incremental tasks. In the federated setting, we simulate a fleet consisting of $K=30$ rovers, maintaining a fixed client population for the duration of training. This client count is chosen so that the Dirichlet partition at $\beta=0.1$ produces a sufficiently heterogeneous label distribution across clients, such that each rover observes only a small subset of classes and the induced gradient conflicts are representative of the stress regime the framework is designed for. A smaller $K$ at the same $\beta$ yields a milder Non-IID partition and would understate the difficulty of the problem. To capture data heterogeneity that results from regional geological differences, we partition the data among clients using a Dirichlet distribution, where the concentration parameter is set to $\beta=0.3$ for MarsScapes and S5Mars, and to $\beta=0.1$ for AI4MARS, so as to reflect the more pronounced cross-mission domain shift present in AI4MARS. For each task, we conduct $R=5$ communication rounds, during which all clients participate in local training. Within each round, every client performs 5 local epochs. The batch size is set to 8, determined by GPU memory limitations when working with high-resolution images. The rehearsal weight is set to $\lambda=1.0$ in Eq.~\eqref{equ6}, giving equal importance to new task learning and memory replay, consistent with the equal-weighting convention in iCaRL~\cite{Rebuffi2017CVPR}. All reported mIoU values are averaged over three independent runs with random seeds $\{42, 82, 132\}$, and the subscripts in Tables~\ref{tab:comparison_marsscapes_5_1}--\ref{tab:comparison_ai4mars_2_1} denote the standard deviation across these runs.

\begin{table}[t]
\centering
\caption{Architecture and parameter counts of the auxiliary modules. ``Size'' is the FP32 on-device storage footprint; all modules remain local.}
\label{tab:aux_modules}
\setlength{\tabcolsep}{1.5mm}
\begin{tabular}{lccc}
\toprule
Module & Architecture & Params & Size \\
\midrule
$\phi_s$ (shallow) & 2-layer MLP & 0.11 M & 0.4 MB \\
$\phi_d$ (deep) & 3-layer MLP & 0.38 M & 1.5 MB \\
$\phi_c$ (head) & Multi-encoder + Fusion & 0.82 M & 3.3 MB \\
$\psi$ (recovery) & 4-encoder + Attn + Decoder & 1.26 M & 5.0 MB \\
\midrule
Total & --- & 2.57 M & 10.2 MB \\
\bottomrule
\end{tabular}
\end{table}

The four auxiliary modules account for a 10.2~MB on-device footprint, about 9.6\% of the baseline segmentation model. None of these parameters enters the communication budget: the LSR generators $\{\omega_s,\omega_d,\omega_c\}$ remain on each rover and are never transmitted, and the recovery function $\psi$ is broadcast once at the end of Task~0 and remains fixed thereafter. The per-round uplink and downlink of our framework during incremental tasks are therefore identical to those of the FedAvg baseline with iCaRL memory.

\begin{table*}[h]
\centering
\setlength{\tabcolsep}{1.8mm}
\caption{Ablation study on Layer-Selective Rehearsal across three Mars terrain datasets.}
\label{tab:ablation_lsr}
\resizebox{\textwidth}{!}{
\begin{tabular}{c|ccc|cccc|cccc|cccc}
\toprule
& \multicolumn{3}{c|}{Components} & \multicolumn{4}{c|}{MarsScapes \cite{Liu2023AA} 5-1} & \multicolumn{4}{c|}{S5Mars \cite{Zhang2024TGRS} 5-1} & \multicolumn{4}{c}{AI4MARS \cite{Swan2021CVPR} 2-1} \\
Settings & $\phi_s$ & $\phi_d$ & $\phi_c$ & 0-7 & 8 & mIoU & Imp. & 0-7 & 8 & mIoU & Imp. & 0-2 & 3 & mIoU & Imp. \\
\midrule
Uniform rehearsal & \multicolumn{3}{c|}{Same $\lambda$} & 29.8 & 12.2 & 26.8 & $\Uparrow$ 3.6 & 41.8 & 21.5 & 38.8 & $\Uparrow$ 3.4 & 34.2 & 11.8 & 28.6 & $\Uparrow$ 4.9 \\
LSR-w/o $\phi_c$ & \cmark & \cmark & \xmark & 23.5 & 10.8 & 21.2 & $\Uparrow$ 9.2 & 34.2 & 19.8 & 31.8 & $\Uparrow$ 10.4 & 26.8 & 9.5 & 22.5 & $\Uparrow$ 11.0 \\
LSR-w/o $\phi_d$ & \cmark & \xmark & \cmark & 31.2 & 12.8 & 28.2 & $\Uparrow$ 2.2 & 43.0 & 23.2 & 40.0 & $\Uparrow$ 2.2 & 37.5 & 13.2 & 31.4 & $\Uparrow$ 2.1 \\
LSR-w/o $\phi_s$ & \xmark & \cmark & \cmark & 32.2 & 13.2 & 30.0 & $\Uparrow$ 0.4 & 44.0 & 24.5 & 41.8 & $\Uparrow$ 0.4 & 39.2 & 14.2 & 33.0 & $\Uparrow$ 0.5 \\
\midrule
\rowcolor{lightgray}
\textbf{Ours} & \cmark & \cmark & \cmark
& \textcolor{deepred}{\textbf{32.5}} & \textcolor{deepred}{\textbf{13.4}} & \textcolor{deepred}{\textbf{30.4}} & --
& \textcolor{deepred}{\textbf{44.4}} & \textcolor{deepred}{\textbf{24.8}} & \textcolor{deepred}{\textbf{42.2}} & --
& \textcolor{deepred}{\textbf{39.7}} & \textcolor{deepred}{\textbf{14.6}} & \textcolor{deepred}{\textbf{33.5}} & -- \\
\bottomrule
\end{tabular}}
\end{table*}

\begin{figure*}[h]
\centering
\includegraphics[width=\textwidth]{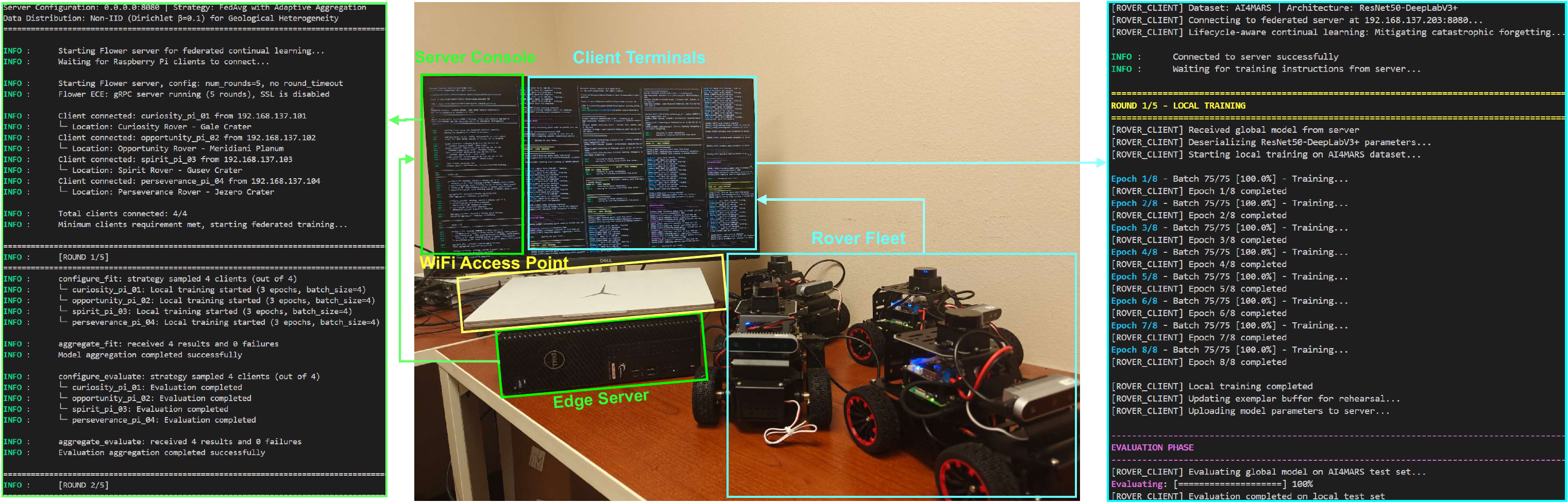}
\caption{Real-world federated continual learning testbed deployment. }
\label{fig:real_world_setup}
\end{figure*}

\subsection{Comparison Performance}

Tables~\ref{tab:comparison_marsscapes_5_1}--\ref{tab:comparison_ai4mars_2_1} present quantitative comparisons against eight representative methods on three Mars terrain datasets under various FCL settings. Our method achieves up to 8.3\% mIoU improvement over the strongest federated baseline and up to 31.7\% over conventional fine-tuning. Under extreme Non-IID conditions ($\beta=0.1$ for AI4MARS), our method outperforms the specialized federated method FBL \cite{Dong2023CVPR} by 3.9\% (33.5\% vs 29.6\%). On AI4MARS under the 2-1 setting, several baselines report zero IoU on the incremental classes (Sand, Big Rock) because the extreme Non-IID partition assigns only a handful of samples per new class to each rover, and the absence of a federated anti-forgetting mechanism causes the newly learned class to be averaged out during aggregation. The Centralized (upper bound) rows at the top of each table report our framework on the pooled data of all $K=30$ clients without federated aggregation, upper-bounding the federated setting (not a joint-training ceiling, since the inter-region distributional shift persists). The $2.2$--$2.7\%$ gap widens with Non-IID severity (largest on AI4MARS, $\beta=0.1$).

Fig.~\ref{fig:recovery_efficiency} reports recovery efficiency when long-term degradation is detected, and also serves as the RKR ablation against Low-Rank Adaptation, Meta-Learning Distillation, Variational Distillation, and Full Retraining. This experiment evaluates restoration on a controlled degradation (training frozen after a specific task with the segmentation head re-initialized) against the Task-0 test set, so the absolute mIoU levels are not directly comparable with the cumulative mIoU in Tables~\ref{tab:comparison_marsscapes_5_1}--\ref{tab:comparison_ai4mars_2_1}. Our RKR function restores
performance in 5-8 epochs: MarsScapes recovers from 14.4\% to 41.4\% in 6 epochs
(Fig.~\ref{fig:recovery_marsscapes}), S5Mars from 21.2\% to 50.6\% in 7 epochs
(Fig.~\ref{fig:recovery_s5mars}), and AI4MARS from 25.4\% to 54.1\% in 6 epochs
(Fig.~\ref{fig:recovery_ai4mars}). In contrast, Full Retraining requires 
18 epochs on MarsScapes, 18 epochs on S5Mars, and 11 epochs on AI4MARS to 
achieve comparable performance. The 2-3$\times$ speedup stems from the recovery 
patterns learned on Task 0, which transfer zero-shot to all subsequent tasks 
without retraining. Other recovery methods show intermediate efficiency: 
Low-Rank Adaptation converges in 10-15 epochs but suffers from capacity
limitations in later tasks, Meta-Learning Distillation requires 8-13 epochs
with per-task episodic optimization overhead, and Variational Distillation
needs 15-20 epochs due to variational approximation constraints.

The trigger threshold $\tau$ is not hand-tuned per dataset but derived as a fixed fraction (80\%) of the Task~0 cumulative mIoU, so the same rule applies across all three datasets and Non-IID regimes. Sweeping $\tau$ over $\{70\%, 80\%, 90\%\}$ on MarsScapes~5-1 yields final mIoU within a $0.8\%$ band ($29.7\%$, $30.4\%$, $30.5\%$) while recovery activations range from $1$ to $4$ (out of $T=4$ incremental tasks); $80\%$ is the communication-efficient operating point, since $90\%$ buys only $0.1\%$ final mIoU at $33\%$ more activations ($4$ vs $3$). Across all main experiments, every trigger activation was followed by a positive cumulative mIoU gain, with diminishing returns for later activations (no spurious activations), and no task with end-of-task cumulative mIoU below $\tau$ went unrecovered (no missed events); the end-of-task evaluation in Algorithm~\ref{alg:training} averages across $R=5$ rounds, which filters single-round noise.

\subsection{Ablation Studies}

To analyze the effectiveness of each component in Layer-Selective Rehearsal, Table~\ref{tab:ablation_lsr} presents ablation experiments. We evaluate four ablation variants: Uniform rehearsal applies the same rehearsal weight $\lambda$ to all network layers without stratification; LSR-w/o $\phi_c$, LSR-w/o $\phi_d$, and LSR-w/o $\phi_s$ indicate our model without the classifier generator, deep layer generator, and shallow layer generator, respectively. When compared with the full LSR framework, all ablation variants degrade performance by $0.4\%\sim11.0\%$ mIoU across datasets, which verifies the importance of stratified protection to address heterogeneous forgetting across network layers.

Notably, the performance gap between uniform rehearsal and full LSR ($3.4\%\sim4.9\%$ mIoU improvement) demonstrates the superiority of stratified protection over conventional uniform strategies. This advantage becomes more noticeable under extreme Non-IID conditions (AI4MARS with $\beta=0.1$), where the stratified approach achieves $4.9\%$ improvement by appropriately allocating protection budgets according to layer-wise forgetting contributions.

\subsection{Real-World Validation}\label{sec:testbed_experiments}

\subsubsection{Experimental Configurations}
\begin{figure*}[h]
	\centering
 	\includegraphics[width=0.95\textwidth]{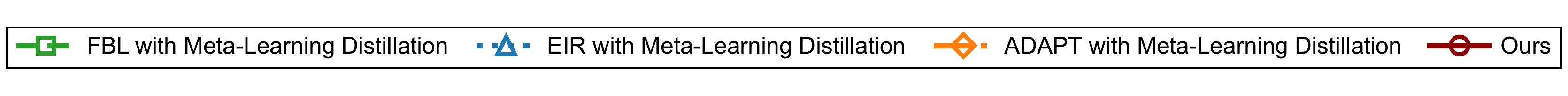}
	
	\vspace{-5pt}
	
	\begin{subfigure}[b]{0.32\textwidth}
		\centering
		\includegraphics[width=\textwidth]{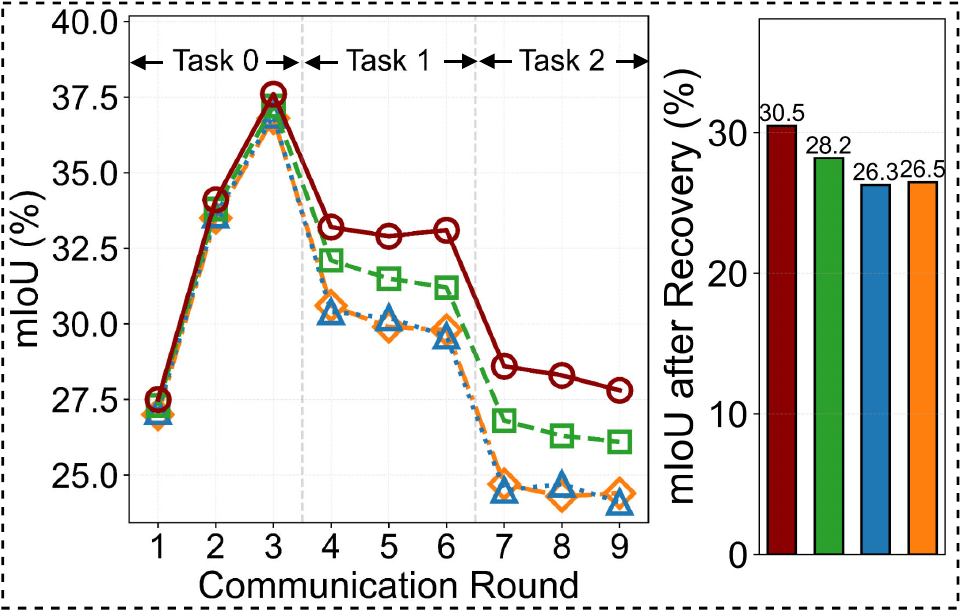}
		\caption{MarsScapes}
		\label{fig:lifecycle_marsscapes}
	\end{subfigure}
	\hfill
	\begin{subfigure}[b]{0.32\textwidth}
		\centering
		\includegraphics[width=\textwidth]{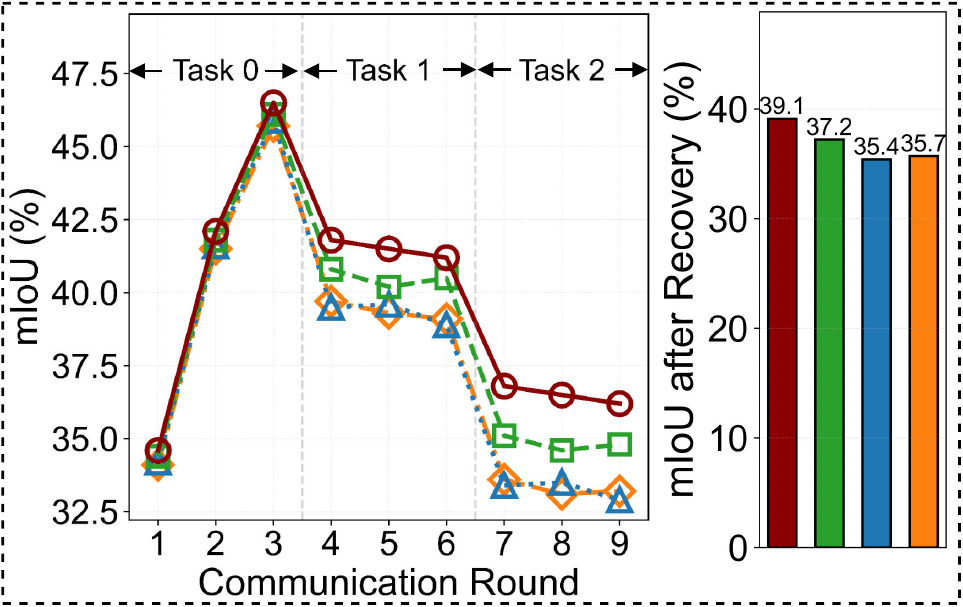}
		\caption{S5Mars}
		\label{fig:lifecycle_s5mars}
	\end{subfigure}
	\hfill
	\begin{subfigure}[b]{0.32\textwidth}
		\centering
		\includegraphics[width=\textwidth]{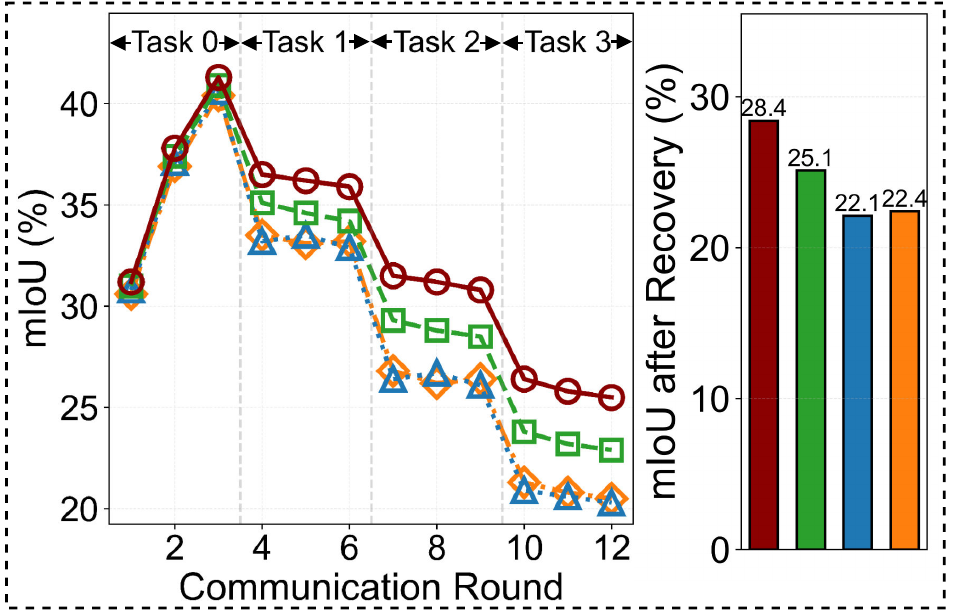}
		\caption{AI4MARS}
		\label{fig:lifecycle_ai4mars}
	\end{subfigure}
	
	\vspace{-2pt}
	\caption{Complete FCL lifecycle evolution across three Mars terrain datasets.}
	\label{fig:lifecycle_evolution}
\end{figure*}

To validate our lifecycle-aware FCL framework in the full learning lifecycle, we construct a physical testbed that enables observation of model evolution from initial task training through long-term degradation and recovery, as illustrated in Fig. \ref{fig:real_world_setup}. The testbed consists of one edge coordination server, four autonomous rover clients, and dedicated network infrastructure operating continuously across sequential learning tasks. The server runs on an Ubuntu 22.04 workstation equipped with an Intel Core i7 processor, 16 GB of memory, and an NVIDIA RTX A2000 GPU, executing the Flower federated learning framework to manage federated aggregation throughout the entire training lifecycle. The rover fleet comprises four MentorPi robotic platforms, each powered by a Raspberry Pi 5 with a Cortex-A76 quad-core processor clocked at 2.4 GHz, representing the computational constraints of the mission lifecycle. Network connectivity is provided through a dedicated laptop serving as a WiFi access point, establishing a 192.168.137.0/24 local network using the 802.11ac protocol. Communication between the server and rover clients is facilitated via gRPC protocol over TCP port 8080, enabling bidirectional transmission of model parameters and training metrics during federated rounds. This deployment architecture allows us to validate both temporal scales of our framework: Layer-Selective Rehearsal operating at the training-time scale during each task stage, and Rapid Knowledge Recovery operating at the long-term scale when cumulative degradation is detected across the extended task sequence. The physical deployment enables end-to-end lifecycle observation under realistic edge computing constraints representative of distributed Mars exploration operations.

The testbed runs with four rover clients, which corresponds to the number of active Mars surface assets in the deployment scenario considered here (Curiosity, Perseverance, and the surface elements of the Mars Sample Return campaign). The testbed therefore directly represents the operational regime that current and near-term missions present, while the $K=30$ simulation covers a more heterogeneous stress regime.

We deploy a pruned DeepLabV3+ \cite{Chen2018ECCV} with ResNet-50 backbone \cite{He2016CVPR} pretrained on ImageNet \cite{Russakovsky2015IJCV}. To accommodate the computational and storage constraints of edge devices, we adapt the input resolution to $256 \times 256$ and construct representative subsets of the dataset while preserving the fundamental characteristics of federated continual learning. For AI4MARS \cite{Swan2021CVPR}, we extract 2,000 images covering all 4 classes (Soil, Bedrock, Sand, Big Rock) and adopt a 1-1-1-1 incremental setting where Task 0 learns the first class, followed by 3 incremental tasks each introducing 1 new class ($T=4$). MarsScapes \cite{Liu2023AA} comprises 1,500 images spanning the first 3 classes (Soil, Bedrock, Sand) with a 1-1-1 configuration ($T=3$), where Task 0 establishes the base model on the first class and subsequent tasks introduce 1 new class each. S5Mars \cite{Zhang2024TGRS} uses 1,200 images across 3 classes (soil, bedrock, rock) with the same 1-1-1 setting ($T=3$). Data partitioning across the four rover clients follows a Dirichlet distribution with $\beta=0.1$ to simulate extreme Non-IID conditions characteristic of geologically heterogeneous Mars terrain. Each rover maintains a local exemplar memory buffer $\mathcal{M}_k$ with capacity allocated heterogeneously to simulate uneven per-client data availability (different traverse durations, on-board prioritization, downlink opportunities) rather than a hardware-level storage limit, with exemplars selected via iCaRL~\cite{Rebuffi2017CVPR}. Training employs SGD optimizer with momentum 0.9, learning rate $1 \times 10^{-3}$, and batch size 2 per client constrained by CPU memory bandwidth. Each task progresses through 3 federated communication rounds, with each round involving 8 local training epochs on the combined dataset of current task data and memory buffer. All four clients participate synchronously in each round. We evaluate cumulative performance using mean Intersection over Union (mIoU) computed over all learned classes.

\subsubsection{Experimental Method}
\begin{figure*}[h]
	\centering
 	\includegraphics[width=0.8\textwidth]{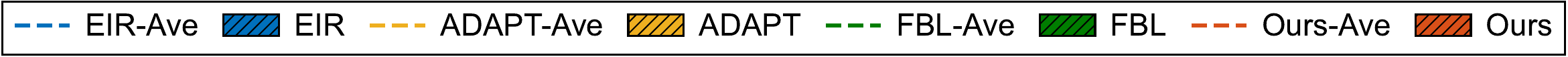}
	
	\vspace{2pt}
	
	\begin{subfigure}[b]{0.32\textwidth}
		\centering
		\includegraphics[width=\textwidth]{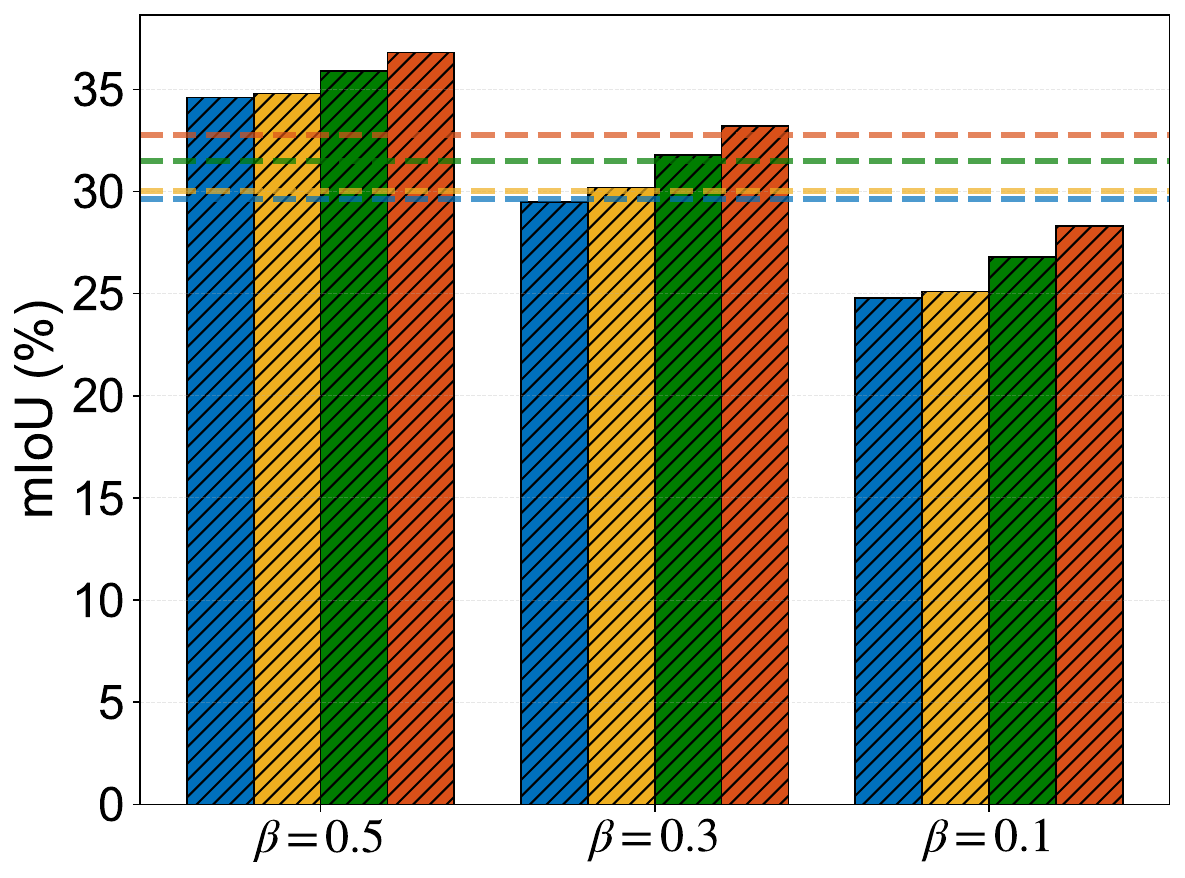}
		\caption{MarsScapes}
		\label{fig:noniid_marsscapes}
	\end{subfigure}
	\hfill
	\begin{subfigure}[b]{0.32\textwidth}
		\centering
		\includegraphics[width=\textwidth]{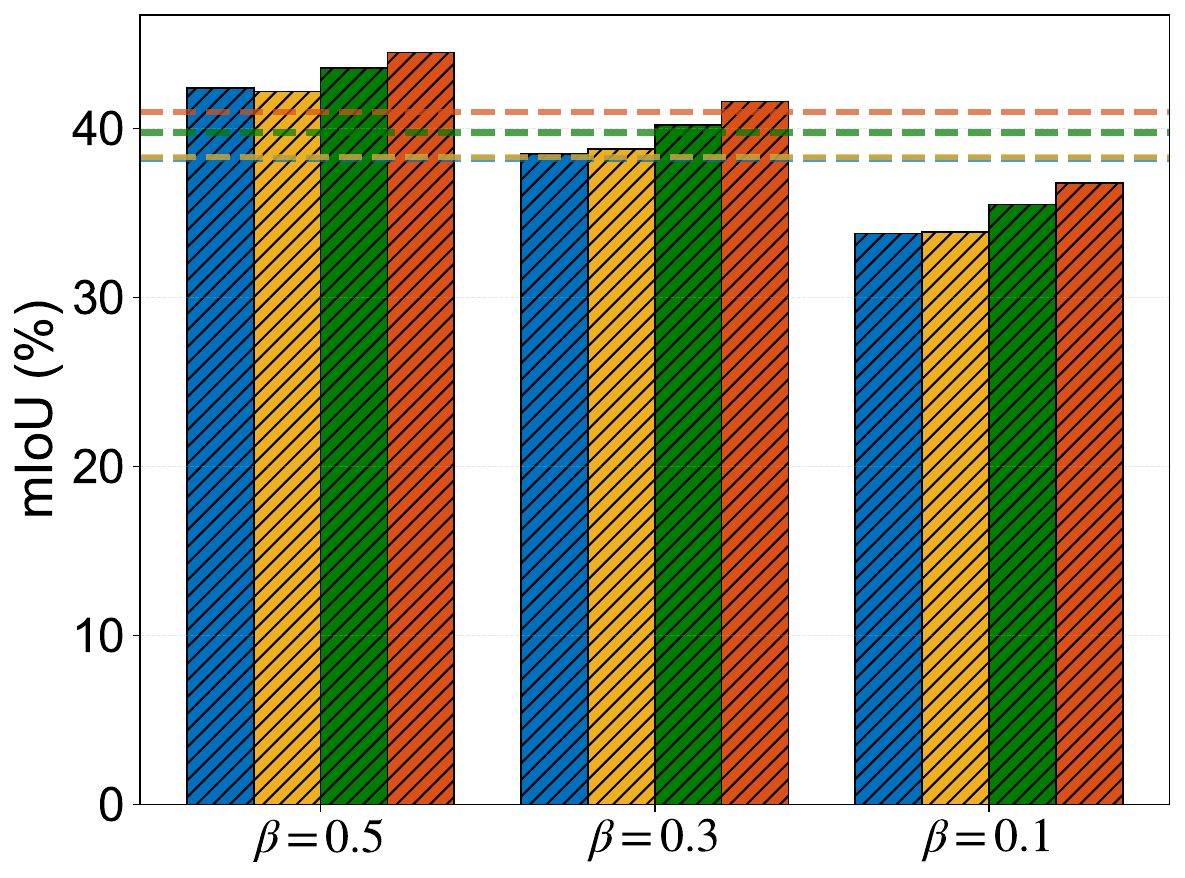}
		\caption{S5Mars}
		\label{fig:noniid_s5mars}
	\end{subfigure}
	\hfill
	\begin{subfigure}[b]{0.32\textwidth}
		\centering
		\includegraphics[width=\textwidth]{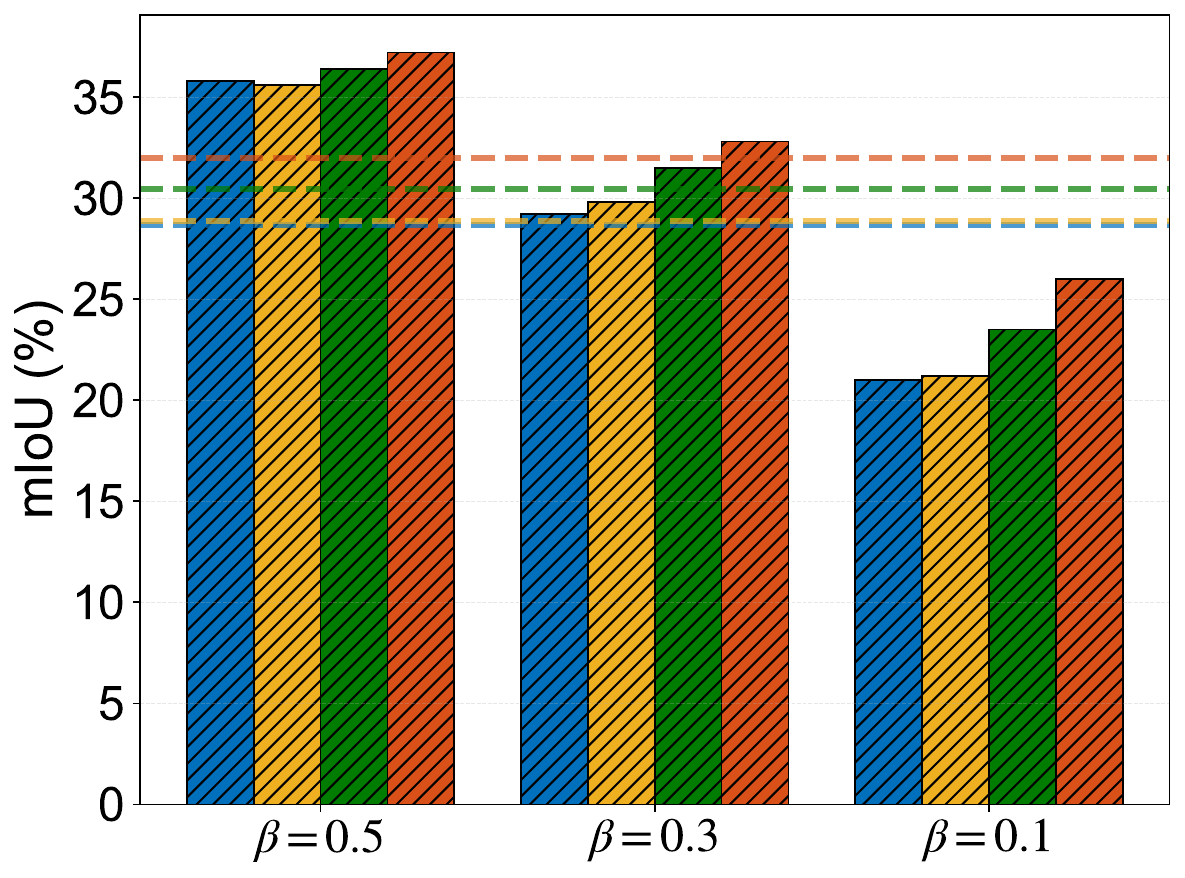}
		\caption{AI4MARS}
		\label{fig:noniid_ai4mars}
	\end{subfigure}
	
	\vspace{-2pt}
	\caption{Performance under varying Non-IID severity across three Mars terrain datasets.}
	\label{fig:noniid_robustness}
\end{figure*}

\begin{figure*}[h]
	\centering
 	\includegraphics[width=0.8\textwidth]{figures/robustness_legend.pdf}
	
	\vspace{2pt}
	
	\begin{subfigure}[b]{0.32\textwidth}
		\centering
		\includegraphics[width=\textwidth]{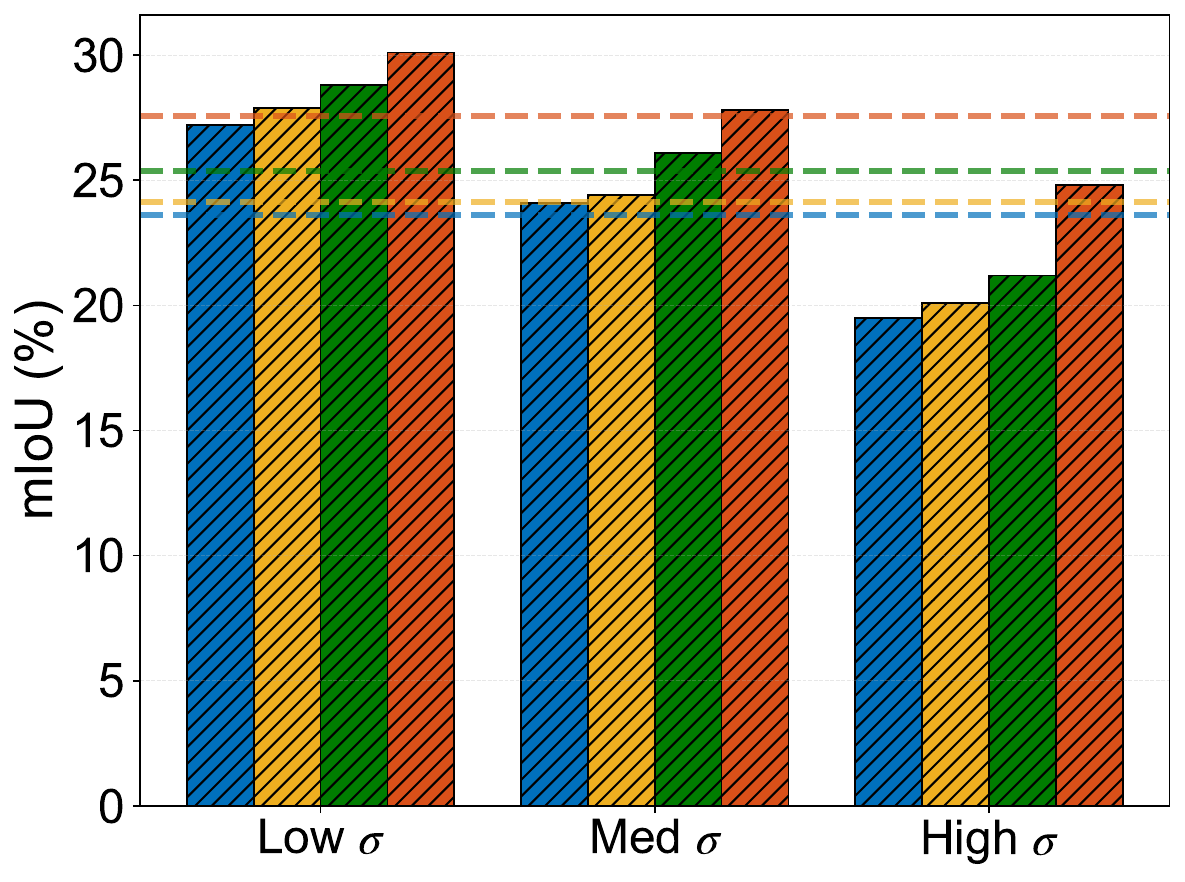}
		\caption{MarsScapes}
		\label{fig:buffer_marsscapes}
	\end{subfigure}
	\hfill
	\begin{subfigure}[b]{0.32\textwidth}
		\centering
		\includegraphics[width=\textwidth]{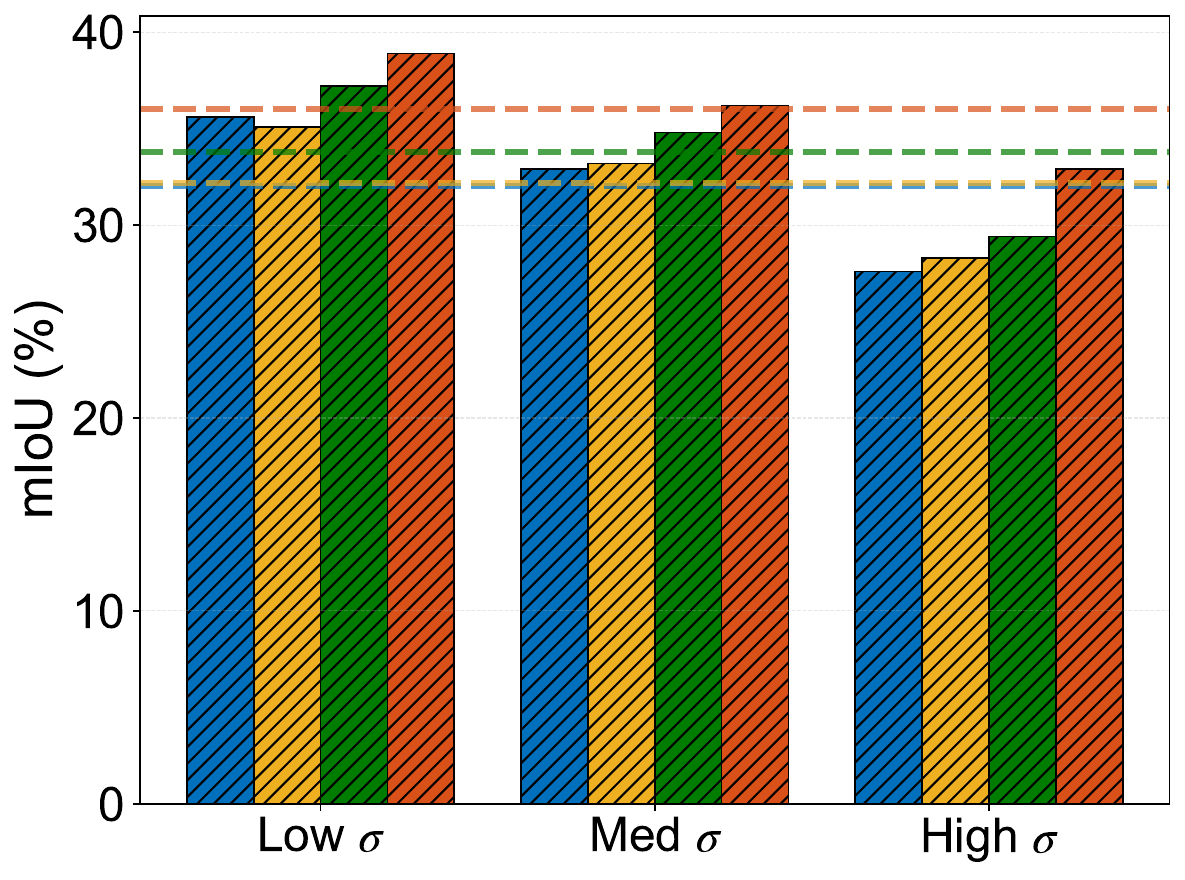}
		\caption{S5Mars}
		\label{fig:buffer_s5mars}
	\end{subfigure}
	\hfill
	\begin{subfigure}[b]{0.32\textwidth}
		\centering
		\includegraphics[width=\textwidth]{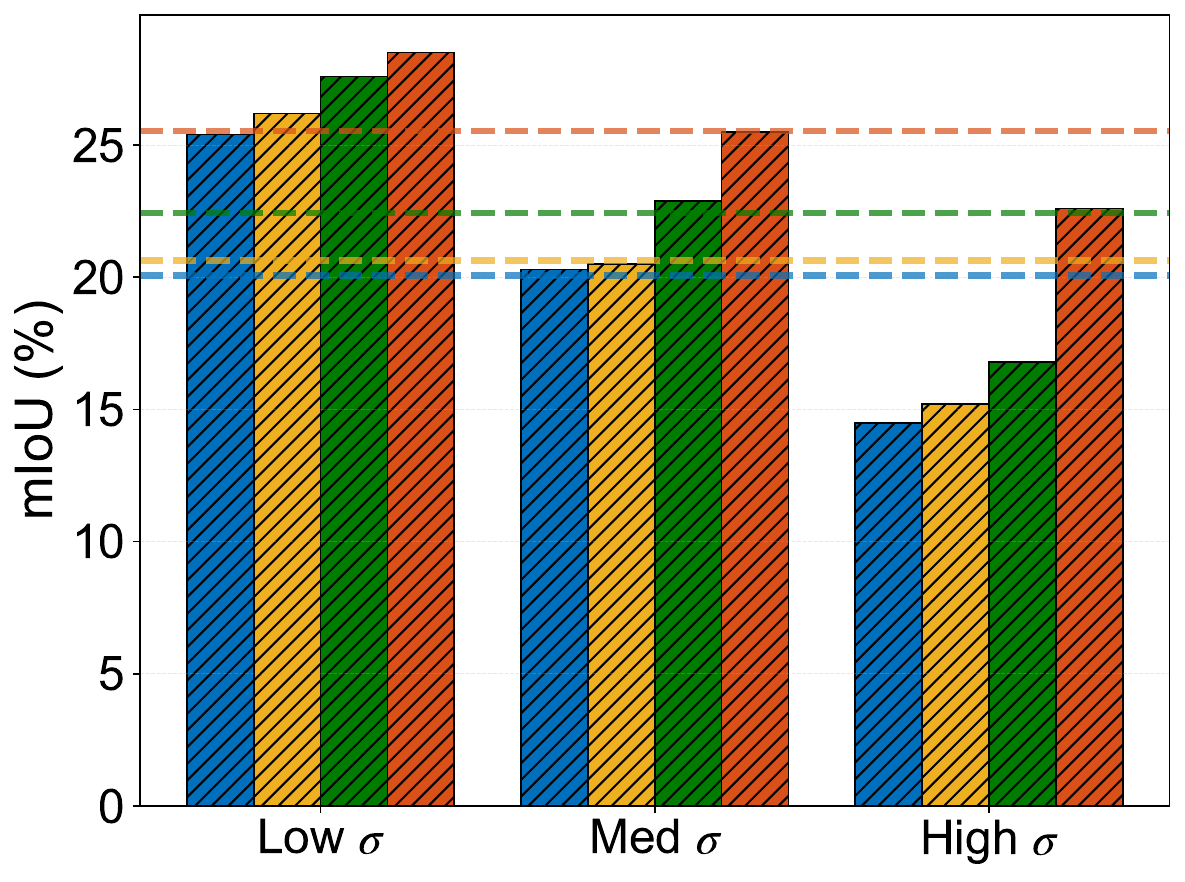}
		\caption{AI4MARS}
		\label{fig:buffer_ai4mars}
	\end{subfigure}
	
	\vspace{-2pt}
	\caption{Performance under varying memory buffer heterogeneity across three Mars terrain datasets.}
	\label{fig:buffer_robustness}
\end{figure*}


While experiments validate our framework under idealistic conditions, the real-world rover testbed introduces challenges closer to practical deployment scenarios: edge-device resource constraints that limit computational capacity, real communication delays in federated coordination, and the complete temporal evolution across extended task sequences. To validate our lifecycle-aware FCL framework under these realistic conditions, we conduct two experiments on the testbed (Fig.~\ref{fig:real_world_setup}): (1) Complete FCL Lifecycle Evolution to observe the coordination between Layer-Selective Rehearsal and Rapid Knowledge Recovery over the full training course, and (2) Edge Deployment Robustness to validate performance under heterogeneous client-side storage constraints characteristic of distributed autonomous systems.

We compare four methods: our proposed approach combining LSR and RKR, and three strong baselines pairing state-of-the-art continual learning methods (FBL~\cite{Dong2023CVPR}, EIR~\cite{Yin2025CVPR}, ADAPT~\cite{Yang2025ICLR}) with Meta-Learning Distillation~\cite{Wang2022CVPR} as the recovery strategy. Those three continual learning methods represent the top-performing approaches from simulation experiments (Tables~\ref{tab:comparison_marsscapes_5_1}-\ref{tab:comparison_ai4mars_2_1}), while Meta-Learning Distillation is selected as the most efficient recovery baseline (Fig.~\ref{fig:recovery_efficiency}). Each method executes the complete task sequence ($T=4$ for AI4MARS, $T=3$ for MarsScapes and S5Mars) with $3$ federated rounds per task. We track the cumulative mIoU evolution across all learned classes over communication rounds to observe training-time forgetting patterns. Recovery is triggered uniformly across all methods after the final task is completed to ensure a fair comparison.

To validate the robustness of our framework under realistic conditions, we conduct edge deployment experiments in two scenarios. In the first scenario, we evaluate performance under varying Non-IID severity by adjusting the Dirichlet concentration parameter across three levels: $\beta=0.5$ (Moderate), $\beta=0.3$ (Severe), and $\beta=0.1$ (Extreme), with all clients maintaining fixed $50$-sample buffers. This tests whether LSR's stratified protection remains effective as gradient conflicts intensify. In the second scenario, we assess adaptability to heterogeneous memory-buffer sizes across the fleet, which simulate uneven per-client data availability rather than hardware storage limits. Three configurations are considered: Low heterogeneity where all four clients maintain $[45, 50, 55, 50]$ sample buffers ($\sigma \approx 4.1$), Medium with $[40, 50, 60, 50]$ buffers ($\sigma \approx 8.2$), and High representing the most uneven case with $[30, 50, 70, 40]$ buffers ($\sigma \approx 16.3$), all under fixed $\beta=0.1$ Non-IID setting. This validates whether LSR's learned generators adapt to varying memory coverage across clients. We compare our method against three state-of-the-art baselines (FBL~\cite{Dong2023CVPR}, EIR~\cite{Yin2025CVPR}, ADAPT~\cite{Yang2025ICLR}) across all configurations. Each configuration executes the complete task sequence ($T=4$ for AI4MARS, $T=3$ for MarsScapes and S5Mars), and measures the cumulative mIoU upon task sequence completion to assess deployment robustness.

\subsubsection{Experimental Results}

Fig.~\ref{fig:lifecycle_evolution} presents the evolution of model performance across three Mars terrain datasets, each evaluated under edge computing constraints. In the left panels, cumulative mIoU is plotted as new tasks are introduced at rounds 4, 7, 10 (AI4MARS) or rounds 4, 7 (MarsScapes, S5Mars), indicating training-time forgetting patterns. Right panels quantify recovery effectiveness after the final task has been completed. We can see that our method consistently achieves between 2.3\% and 3.3\% higher mIoU before recovery, and between 2.0\% and 3.3\% higher mIoU after recovery, when compared to FBL with Meta-Learning Distillation.

Fig.~\ref{fig:noniid_robustness} examines how each method responds to increasing levels of Non-IID severity, with $\beta$ values set to 0.5, 0.3, and 0.1. As the degree of data heterogeneity increases, all methods experience a decline in performance. However, our approach is able to maintain stronger results, which can be attributed to the stratified protection provided by LSR. In particular, when $\beta=0.1$, representing the most severe heterogeneity, our method achieves between 0.8\% and 2.5\% higher mIoU than FBL. The outcome demonstrates the effectiveness of layer-wise adaptive rehearsal in mitigating the impact of Non-IID data.

Fig.~\ref{fig:buffer_robustness} investigates the robustness of each method when faced with varying degrees of memory buffer heterogeneity, where the standard deviation $\sigma$ is approximately 4.1 (Low), 8.2 (Medium), and 16.3 (High). As the imbalance in storage allocation becomes more pronounced, baseline methods show a decline in performance, which is primarily due to insufficient memory coverage. In contrast, our method is able to maintain between 0.9\% and 5.8\% higher mIoU than FBL under conditions of high heterogeneity. The improvement is a result of the adaptive generators used in LSR, which help to mitigate the negative effects of uneven memory distribution.

We note three limitations of the present testbed. The Raspberry Pi 5 client is a proxy for radiation-hardened flight CPUs rather than an exact match; the WiFi 802.11ac link is a proxy for UHF relay links and does not capture the contact-window structure of deep-space operations; and the four-client scale does not directly verify behavior at $K\!\geq\!30$, which is the role of the $K=30$ simulation reported earlier in this section. We report these limitations so that the testbed results are read in their intended scope.

Real missions can present distributional complexity that goes beyond the non-overlapping class-incremental setting considered in our experiments, including domain-incremental conditions (same label set across geological regions or illumination conditions) and class-overlap conditions (a label reappearing with a different sub-distribution). The framework is structurally compatible with both: the LSR generators are conditioned on the current gradient and memory features and can absorb a shift in the input distribution as a new gradient regime, and the RKR episodic meta-training on Task~0 can be augmented with domain-shift perturbations to prepare the recovery function for domain-incremental episodes. No public Mars benchmark currently covers the domain-incremental setting required to quantify this extension.

\section{Conclusion}
\label{sec:conclusion}

In this paper, we have investigated Federated Continual Learning (FCL) in mobile autonomous systems, with Mars terrain segmentation as the primary evaluation scenario. We have identified heterogeneous forgetting dynamics in deep network layers and formulated the dual-timescale knowledge-degradation problem induced by repeated federated aggregation. To address these challenges, we have designed a lifecycle-aware FCL framework that coordinates Layer-Selective Rehearsal (LSR) for training-time gradient conflicts and Rapid Knowledge Recovery (RKR) for long-term cumulative drift. Our evaluation results, derived from both simulations on three Mars terrain datasets and a physical rover testbed, demonstrate that the proposed framework achieves up to 8.3\% mIoU improvement over the strongest federated baseline and up to 31.7\% over conventional fine-tuning, while maintaining robustness under extreme Non-IID conditions and resource heterogeneity. The system-level validation on heterogeneous edge devices confirms the practical applicability of our framework for distributed autonomous missions.
\section*{Acknowledgments}
This work was supported by the National Science Foundation under CNS-2348422.

\bibliographystyle{IEEEtran}
\bibliography{reference}  

\end{document}